%% file: preck.tex
\documentclass{article}
\usepackage{times}
\usepackage{graphicx}
\usepackage{subfigure} 
\usepackage{natbib}
\usepackage{fullpage}

\usepackage{algorithmic}
\usepackage{algorithm}
\usepackage{hyperref}
\usepackage{amsmath,amsthm,amssymb}
\usepackage{xspace}
\usepackage{color}
\usepackage{mdframed}
\usepackage{epstopdf}
\usepackage{dblfloatfix}
\usepackage{lipsum}
\include{defs}

\title{Surrogate  Functions for Maximizing Precision at the Top}
\author{Purushottam Kar\\Microsoft Research India\\t-purkar@microsoft.com \and Harikrishna Narasimhan$^\ast$\\Indian Institute of Science\\harikrishna@csa.iisc.ernet.in \and Prateek Jain\\Microsoft Research India\\prajain@microsoft.com}

\begin{document}

\maketitle

\begin{abstract} 
The problem of maximizing precision at the top of a ranked list, often dubbed Precision@k (\preck), finds relevance in myriad learning applications such as ranking, multi-label classification, and learning with severe label imbalance. However, despite its popularity, there exist significant gaps in our understanding of this problem and its associated performance measure.

The most notable of these is the lack of a convex upper bounding surrogate for \preck. We also lack scalable perceptron and stochastic gradient descent algorithms for optimizing this performance measure. In this paper we make key contributions in these directions. At the heart of our results is a family of truly upper bounding surrogates for \preck. These surrogates are motivated in a principled manner and enjoy attractive properties such as consistency to \preck under various natural margin/noise conditions.

These surrogates are then used to design a class of novel perceptron algorithms for optimizing \preck with provable mistake bounds. We also devise scalable stochastic gradient descent style methods for this problem with provable convergence bounds. Our proofs rely on novel uniform convergence bounds which require an in-depth analysis of the structural properties of \preck and its surrogates. We conclude with experimental results comparing our algorithms with state-of-the-art cutting plane and stochastic gradient algorithms for maximizing \preck.
\end{abstract}

\let\thefootnote\relax\footnotetext{$^\ast$ Work done while H.N. was an intern at Microsoft Research India, Bangalore.}

\input{intro}
\input{related-works}
\input{formulation}
\input{new-surrogate_new}
\input{method}
\input{genbound}
\input{exps}

\section*{Acknowledgments}
HN thanks support from a Google India PhD Fellowship.

\bibliographystyle{plainnat}
\bibliography{refs}

\appendix
\onecolumn

\allowdisplaybreaks

\input{app}
\input{app_exps}

\end{document}

%% file: defs.tex
\newtheorem{lem}{Lemma}
\newtheorem{thm}[lem]{Theorem}
\newtheorem{cor}[lem]{Corollary}
\newtheorem{clm}[lem]{Claim}
\newtheorem{defn}[lem]{Definition}

\makeatletter
\newcommand{\newreptheorem}[2]{\newtheorem*{rep@#1}{\rep@title} 
\newenvironment{rep#1}[1]{\def\rep@title{#2 \ref*{##1}}\begin{rep@#1}}{\end{rep@#1}}
}
\makeatother

\newreptheorem{lem}{Lemma}
\newreptheorem{thm}{Theorem}
\newreptheorem{clm}{Claim}

\newcommand{\preck}{\textup{\textrm{prec@}k}\xspace}
\newcommand{\preckappa}{\textup{\textrm{prec@}$\kappa$}\xspace}

\newcommand{\struct}{\text{struct}\xspace}
\newcommand{\ramp}{\text{ramp}\xspace}
\newcommand{\avg}{\text{avg}\xspace}
\newcommand{\mx}{\text{max}\xspace}

\newcommand{\perckavg}{{\sc Perceptron@k-avg}\xspace}
\newcommand{\perckmax}{{\sc Perceptron@k-max}\xspace}
\newcommand{\sgdk}{{\sc SGD@k-avg}\xspace}
\newcommand{\sgdkmax}{{\sc SGD@k-max}\xspace}
\newcommand{\FN}{\text{FN}}
\newcommand{\spmb}{{\bfseries 1PMB}\xspace}

\newcommand{\R}{{\mathbb R}}

\newcommand{\V}{{\cal V}}

\newcommand{\X}{{\cal X}}

\renewcommand{\L}{{\cal L}}
\newcommand{\T}{{\mathbb T}}

\newcommand{\W}{{\cal W}}

\newcommand{\K}{{\cal K}}

\newcommand{\ZZ}{{\cal Z}}

\newcommand{\bc}[1]{\left\{{#1}\right\}}
\newcommand{\br}[1]{\left({#1}\right)}
\newcommand{\bs}[1]{\left[{#1}\right]}
\newcommand{\abs}[1]{\left| {#1} \right|}
\newcommand{\norm}[1]{\left\| {#1} \right\|}
\newcommand{\ceil}[1]{\left\lceil #1 \right\rceil}

\renewcommand{\O}[1]{{\cal O}\br{{#1}}}
\newcommand{\softO}[1]{\tilde{\cal O}\br{{#1}}}
\newcommand{\Om}[1]{\Omega\br{{#1}}}

\newcommand{\ip}[2]{\left\langle{#1},{#2}\right\rangle}

\renewcommand{\vec}[1]{{\mathbf{#1}}}

\newcommand{\vecX}{\vec{X}}

\newcommand{\vz}{\vec{0}}

\newcommand{\x}{\vec{x}}
\newcommand{\y}{\vec{y}}
\newcommand{\z}{\vec{z}}
\newcommand{\g}{\vec{g}}

\newcommand{\w}{\vec{w}}
\renewcommand{\v}{\vec{v}}
\newcommand{\hx}{\hat\x}
\newcommand{\hy}{\hat\y}
\newcommand{\hz}{\hat\z}
\newcommand{\ty}{\tilde\y}

\newcommand{\Ind}[1]{{\mathbb I}\bs{{#1}}}

\renewcommand{\>}{\rightarrow}
\newcommand{\<}{\leftarrow}

%% file: intro.tex
\section{Introduction}
\label{sec:intro}
Ranking a given set of points or labels according to their relevance forms the core of several real-life learning systems. For instance, in classification problems with a rare-class as is the case in spam/anomaly detection, the goal is to rank the given emails/events according to their likelihood of being from the rare-class (spam/anomaly). Similarly, in multi-label classification problems, the goal is to rank the labels according to their likelihood of being relevant to a data point \cite{TsoumakasK07}. 

The ranking of items at the top is of utmost importance in these applications and several performance measures, such as Precision@k, Average Precision and NDCG have been designed to promote accuracy at top of ranked lists. Of these, the Precision@k (\preck) measure is especially popular in a variety of domains. Informally, \preck counts the number of relevant items in the top-k positions of a ranked list and is widely used in domains such as binary classification \cite{Joachims05}, multi-label classification \cite{PrabhuV14} and ranking \cite{LeSmola07}.

Given its popularity, \preck has received attention from algorithmic, as well as learning theoretic perspectives. However, there remain specific deficiencies in our understanding of this performance measure. In fact, to the best of our knowledge, there is only one known convex surrogate function for \preck, namely, the \emph{struct-SVM} surrogate due to \cite{Joachims05} which, as we reveal in this work, is not an upper bound on \preck in general, and need not recover an optimal ranking even in strictly separable settings. 

Our aim in this paper is to develop efficient algorithms for optimizing \preck for ranking problems with binary relevance levels. Since the intractability of binary classification in the agnostic setting \cite{GuruswamiR09} extends to \preck, our goal would be to exploit natural notions of \emph{benign-ness} usually observed in natural distributions to overcome such intractability results.

\subsection{Our Contributions}
We make several contributions in this paper that both, give deeper insight into the \preck performance measure, as well as provide scalable techniques for optimizing it.

\textbf{Precision@k margin}: motivated by the success of margin-based frameworks in classification settings, we develop a family of margin conditions appropriate for the \preck problem. Recall that the \preck performance measure counts the number of relevant items at the top $k$ positions of a ranked list. The simplest of our margin notions, that we call the \emph{weak $(k,\gamma)$-margin}, is said to be present if a privileged set of $k$ relevant items can be separated from all irrelevant items by a margin of $\gamma$.  This is the least restrictive margin condition that allows for a perfect ranking w.r.t \preck. Notably, it is much less restrictive than the binary classification notion of margin which requires all relevant items to be separable from all irrelevant items by a certain margin. We also propose two other notions of margin suited to our perceptron algorithms.

\textbf{Surrogate functions for \preck}: we design a family of three novel surrogates for the \preck performance measure. Our surrogates satisfy two key properties. Firstly they always upper bound the \preck performance measure so that optimizing them promotes better performance w.r.t \preck. Secondly, these surrogates satisfy \emph{conditional consistency} in that they are consistent w.r.t. \preck under some noise condition. We show that there exists a one-one relationship between the three \preck margin conditions mentioned earlier and these three surrogates so that each surrogate is consistent w.r.t. \preck under one of the margin conditions. Moreover, our discussion reveals that the three surrogates, as well as the three margin conditions, lie in a concise hierarchy.

\textbf{Perceptron and SGD algorithms}: using insights gained from the previous analyses, we design two perceptron-style algorithms for optimizing \preck. Our algorithms can be shown to be a natural extension of the classical perceptron algorithm for binary classification \cite{rosenblatt58a}. Indeed, akin to the classical perceptron, both our algorithms enjoy mistake bounds that reduce to crisp convergence bounds under the margin conditions mentioned earlier. We also design a mini-batch-style stochastic gradient descent algorithm for optimizing \preck.

\textbf{Learning theory}: in order to prove convergence bounds for the SGD algorithm, and online-to-batch conversion bounds for our perceptron algorithms, we further study \preck and its surrogates and prove uniform convergence bounds for the same. These are novel results and require an in-depth analysis into the involved structure of the \preck performance measure and its surrogates. However, with these results in hand, we are able to establish crisp convergence bounds for the SGD algorithm, as well as generalization bounds for our perceptron algorithms.

{\bf Paper Organization}: Section~\ref{sec:formulation} presents the problem formulation and sets up the notation. Section~\ref{sec:new-surrogate} introduces three novel surrogates and margin conditions for \preck and reveals the interplay between these with respect to consistency to \preck. Section~\ref{sec:perceptron} presents two perceptron algorithms for \preck and their mistake bounds, as well as a mini-batch SGD-based algorithm. Section~\ref{sec:genbound} discusses uniform convergence bounds for our surrogates and their application to convergence and online-to-batch conversion bounds for our the perceptron and SGD-style algorithms. We conclude with empirical results in Section~\ref{sec:exps}.


%% file: related-works.tex
\subsection{Related Work}
There has been much work in the last decade in designing algorithms for bipartite ranking problems. While the earlier methods  for this problem, such as RankSVM, focused on optimizing pair-wise ranking accuracy \cite{ranksvm1, ranksvm2, rankboost, ranknet}, of late, there has been enormous interest in performance measures that promote good ranking performance at the top portion of the ranked list, and in ranking methods that directly optimize these measures \cite{Clemencon07,pnorm,infpush,Boyd+12,NarasimhanA13a,NarasimhanA13b,Nan+14}.

In this work, we focus on one such evaluation measure -- Precision@k, which is widely used in practice. The only prior algorithms that we are aware of that directly optimize this performance measure are a structural SVM based cutting plane method due to \cite{Joachims05}, and an efficient stochastic implementation of the same due to \cite{KarN014}. However, as pointed out earlier, the convex surrogate used in these methods is not well-suited for \preck. 

It is also important to note that the bipartite ranking setting considered in this work is different from other popular forms of ranking such as subset or list-wise ranking settings, which arise in several information retrieval applications, where again there has been much work in optimizing performance measures that emphasize on accuracy at the top (e.g. NDCG) \cite{ndcg09, listnet, svmmap, LeSmola07, ChakrabartiKSB08, HRV14}. There has also been some recent work on perceptron style ranking methods for list-wise ranking problems \cite{ChAm14}, but these methods are tailored to optimize the NDCG and MAP measures, which are different from the \preck measure that we consider here. Other less related works include online ranking algorithms for optimizing ranking measures in an adversarial setting with limited feedback \cite{ChAm15}.


%% file: formulation.tex
\section{Problem Formulation and Notation}
\label{sec:formulation}
We will be presented with a set of labeled points $(\x_i,y_i),\ldots, (\x_n,y_n)$, where $\x_i \in \X$ and $y_i \in \{0,1\}$. We shall use $\vecX$ to denote the entire dataset, $\vecX_+$ and $\vecX_-$ to denote the set of positive and negatively (null) labeled points, and $\y \in \{0,1\}^n$ to denote the label vector. $\z = (\x,y)$ shall denote a labeled data point. Our results readily extend to multi-label and ranking settings but for sake of simplicity, we focus only on bipartite ranking problems, where the goal is to rank (a subset of) positive examples above the negative ones. 

Given $n$ labeled data points $\z_1,\ldots,\z_n$ and a scoring function $s: \X \> \R$, let $\sigma_s \in S_n$ be the permutation that sorts points according to the scores given by $s$ i.e. $s(\x_{\sigma_s(i)}) \geq s(\x_{\sigma_s(j)})$ for $i \leq j$. The Precision@k measure for this scoring function can then be expressed as:
 
\begin{align}
\label{eq:preck-def}
\preck(s;\ \z_1,\ldots,\z_n) = \sum_{i=1}^k (1- \y_{\sigma_s(i)}).
\end{align}
Note that the above is a ``loss'' version of the performance measure which penalizes any top-$k$ ranked data points that have a null label. For simplicity, we will use the abbreviated notation $\preck(s) := \preck(s;\ \z_1,\ldots,\z_n)$. We will also use the shorthand $s_i = s(\x_i)$. For any label vectors $\y',\y'' \in \{0,1\}^n$, we define
\begin{equation}
\begin{aligned}
\Delta(\y',\y'') &= \sum_{i=1}^n(1-\y_i')\y_i'',\\
K(\y',\y'') &= \sum_{i=1}^n\y_i'\y_i''.\label{eq:del}
\end{aligned}
\end{equation}
Let $n_+(\y') = K(\y',\y') = \norm{\y'}_1$ denote the number of positives in the label vector $\y'$ and $n_+ = n_+(\y)$ denote the number of actual positives. Let $\y^{(s,k)}$ be the label vector that assigns the label $1$ only to the top $k$ ranked items according to the scoring function $s$. That is, $\y^{(s,k)}_i = 1$ if $\text{if } \sigma^{-1}_s(i) \leq k$ and $0$ otherwise. It is easy to verify that for any scoring function $s$, $\Delta(\y,\y^{(s,k)}) = \preck(s)$.


%% file: new-surrogate_new.tex
\section{A Family of Novel Surrogates for \preck}
\label{sec:new-surrogate}
As \preck is a non-convex loss function that is hard to optimize directly, it is natural to seek surrogate functions that act as a good proxy for \preck. There will be two properties that we shall desire of such a surrogate:
\begin{enumerate}
	\item \textbf{Upper Bounding Property}: the surrogate should \emph{upper bound} the \preck loss function, so that minimizing the surrogate promotes small \preck loss.
	\item \textbf{Conditional Consistency}: under some regularity assumptions, optimizing the surrogate should yield an optimal solution for \preck as well.
\end{enumerate}
Motivated by the above requirements, we develop a family of surrogates which upper bound the \preck loss function and are consistent to it under certain margin/noise conditions. We note that the results of \cite{CalauzenesUG12} that negate the possibility of consistent convex surrogates for ranking performance measures do not apply to our results since they are neither stated for \preck, nor do they negate the possibility of conditional consistency.

It is notable that the seminal work of \cite{Joachims05} did propose a convex surrogate for \preck, that we refer to as $\ell^\struct_\preck(\cdot)$. However, as the discussion below shows, this surrogate is not even an upper bound on \preck let alone be consistent to it. Understanding the reasons for the failure of this surrogate would be crucial in designing our own.

\subsection{The Curious Case of $\ell^\struct_\preck(\cdot)$}
The $\ell^\struct_\preck(\cdot)$ surrogate is a part of a broad class of surrogates called {\em struct-SVM} surrogates that are designed for structured output prediction problems that can have exponentially large output spaces \cite{Joachims05}. Given a set of $n$ labeled data points, $\ell^\struct_\preck(\cdot)$ is defined as
\begin{align}
\max_{\substack{\hy\in\bc{0,1}^n\\\norm{\hy}_1 = k}}\bc{\Delta(\y,\hy) + \sum_{i=1}^n\br{\hy_i-\y_i}s_i}.
\end{align}
The above surrogate penalizes a scoring function if there exists a set of $k$ points with large scores (i.e. the second term is large) which are actually negatives (i.e. the first term is large). However, since the candidate labeling $\hy$ is restricted to labeling just $k$ points as positive whereas the true label vector $\y$ has $n_+$ positives, in cases where $n_+ > k$, a non-optimal candidate labeling $\hy$ can exploit the remaining $n_+-k$ labels to hide the high scoring negative points, thus confusing the surrogate function. This indicates that this surrogate may not be an upper bound to \preck. We refer the reader to Appendix~\ref{sec:old-surrogate} for an explicit example where, not only does this surrogate not upper bound \preck, but more importantly, minimizing $\ell^\struct_\preck(\cdot)$ does not produce a model that is optimal for \preck, even in separable settings where all positives points are separated from negatives by a margin.

In the sequel, we shall propose three surrogates, all of which are consistent with \preck under various noise/margin conditions. The surrogates, as well as the noise conditions, will be shown to form a hierarchy.

\subsection{The Ramp Surrogate $\ell^\ramp_\preck(\cdot)$}
The key to maximizing \preck in a bipartite ranking setting is to select a subset of $k$ relevant items and rank them at the top $k$ positions. This can happen iff the \emph{top ranked} $k$ relevant items are not outranked by any irrelevant item. Thus, a surrogate must penalize a scoring function that assigns scores to irrelevant items that are higher than those of the top ranked relevant items. Our {\em ramp} surrogate $\ell^\ramp_\preck(s)$ implicitly encodes this strategy:
\begin{equation}
\label{eq:ramp_surr}
\max_{\norm{\hy}_1=k}\bc{\Delta(\y,\hy) + \sum_{i=1}^n\hy_is_i} - \underbrace{\max_{\substack{\norm{\ty}_1 = k\\K(\y,\tilde\y) = k}}\sum_{i=1}^n\tilde\y_is_i}_{(P)}.
\end{equation}
The term $(P)$ contains the sum of scores of the $k$ highest scoring positives. Note that $\ell^\ramp_\preck(\cdot)$ is similar to the ``ramp'' losses for binary classification \cite{DoLTCS08}. We now show that $\ell^\ramp_\preck(\cdot)$ is indeed an upper bounding surrogate for \preck.
\begin{clm}
\label{clm:upper-bd-tight}
For any $k \leq n_+$ and scoring function $s$, we have $\ell^\ramp_\preck(s) \geq \preck(s)$. Moreover, if $\ell^\ramp_\preck(s) \leq \xi$ for a given scoring function $s$, then there necessarily exists a set $S \subset [n]$ of size at most $k$ such that for all $\|\hy\|_1 = k$, we have $\sum_{i \in S}s_i \geq \sum_{i=1}^n\hy_is_i + \Delta(\y,\hy) - \xi.$
\end{clm}
Proofs for this section are deferred to Appendix~\ref{app:prf_surrogate}. We can show that this surrogate is conditionally consistent as well. To do so, we introduce the notion of {\em weak $(k,\gamma)$-margin}.
\begin{defn}[\emph{Weak} $(k,\gamma)$-margin]
\label{defn:weak-k-margin}
A set of $n$ labeled data points satisfies the \emph{weak} $(k,\gamma)$-margin condition if for some scoring function $s$ and set $S_+ \subseteq \vecX_+$ of size $k$,
\[
\min_{i \in S_+} s_i - \max_{j : \y_j = 0} s_j \geq \gamma.
\]
Moreover, we say that the function $s$ \emph{realizes} this margin. We abbreviate the \emph{weak} $(k,1)$-margin condition as simply the \emph{weak} $k$-margin condition.
\end{defn}
Clearly, a dataset has a {\em weak} $(k,\gamma)$-margin iff there exist some $k$ positive points that substantially outrank all negatives. Note that this notion of margin is strictly weaker than the standard notion of margin for binary classification as it allows all but those $k$ positives to be completely mingled with the negatives. Moreover, this seems to be one of the most natural notions of margin for \preck. The following lemma establishes that $\ell^\ramp_\preck(\cdot)$ is indeed consistent w.r.t. \preck under the {\em weak} $k$-margin condition.
\begin{clm}
\label{clm:consistent-weak}
For any scoring function $s$ that realizes the \emph{weak} $k$-margin over a dataset, $\ell^\ramp_\preck(s) = \preck(s) = 0$.
\end{clm}
This suggests that $\ell^\ramp_\preck(\cdot)$ is not only a tight surrogate, but tight at the optimal scoring function, i.e. $\preck(s) = 0$; this along with upper bounding property implies consistency. However, it is also a non-convex function due to the term $(P)$. To obtain convex surrogates, we perform relaxations on this term by first rewriting it as follows:
\begin{equation}
\label{eq:pq}
(P) = \sum_{i=1}^n\y_is_i - \underbrace{\min_{\substack{\tilde\y \preceq \y \\ \norm{\tilde\y}_1 = n_+ - k}}\sum_{i=1}^n\tilde\y_is_i}_{(Q)},
\end{equation}
where $\tilde\y \preceq \y$ implies that $\y_i = 0 \Rightarrow \tilde\y_i = 0$. Thus, to convexify the surrogate $\ell^\ramp_\preck(\cdot)$, we need to design a convex upper bound on $(Q)$. Notice that the term $(Q)$ contains the sum of the scores of the $n_+ - k$ lowest ranked positive data points. This can be readily upper bounded in several ways which give us different surrogate functions.

\subsection{The Max Surrogate $\ell^\mx_\preck(\cdot)$}
An immediate convex upper bound on $(Q)$ is obtained by replacing the sum of scores of the $n_+ - k$ lowest ranked positives with those of the highest ranked ones as follows: $(Q) \leq \max_{\substack{\tilde\y \preceq (1-\hy)\cdot\y \\ \norm{\tilde\y}_1 = n_+ - k}}\sum_{i=1}^n\tilde\y_is_i$, which gives us the $\ell^\mx_\preck(s)$ surrogate defined below: 
\vskip-3ex
\small
\begin{equation}
\label{eq:max_surr}
\max_{\norm{\hy}_1=k}\bc{\Delta(\y,\hy) + \sum_{i=1}^n(\hy_i - \y_i)s_i+ \max_{\substack{\tilde\y \preceq (1-\hy)\cdot\y \\ \norm{\tilde\y}_1 = n_+ - k}}\sum_{i=1}^n\tilde\y_is_i}.
\end{equation}
\normalsize
The above surrogate, being a point-wise maximum over convex functions, is convex, as well as an upper bound on $\preck(s)$ since it upper bounds $\ell^\ramp_\preck(s)$. This surrogate can also be shown to be consistent w.r.t. \preck under the {\em strong} $\gamma$-margin condition defined below for $\gamma = 1$.
\begin{defn}[\emph{Strong} $\gamma$-margin]
\label{defn:strong-k-margin}
A set of $n$ labeled data points satisfies the $\gamma$-\emph{strong} margin condition if for some scoring function $s$, $\min_{i : \y_i = 1} s_i - \max_{j : \y_j = 0} s_j \geq \gamma$.
\end{defn}
We notice that the strong margin condition is actually the standard notion of binary classification margin and hence much stronger than the {\em weak} $(k,\gamma)$-margin condition. It also does not incorporate any elements of the \preck problem. This leads us to look for tighter convex relaxations to the term (Q) that we do below.

\subsection{The Avg Surrogate $\ell^\avg_\preck(\cdot)$}
A tighter upper bound on $(Q)$ can be obtained by replacing $(Q)$ by the average score of the false negatives.
Define $C(\hy) = \frac{n_+ - K(\y,\hy)}{n_+ - k}$ and consider the relaxation $(Q)\leq \frac{1}{C(\hy)}\sum_{i=1}^n(1 - \hy_i)\y_is_i$. Combining this  with \eqref{eq:ramp_surr}, we get a new convex surrogate $\ell^\avg_\preck(s)$ defined as:
\vskip-3ex
\small
\begin{equation}
\label{eq:avg_surr}
\max_{\norm{\hy}_1=k}\bc{\Delta(\y,\hy) + \sum_{i=1}^ns_i(\hy_i-\y_i) + \frac{1}{C(\hy)}\sum_{i=1}^n(1 - \hy_i)\y_is_i}.
\end{equation}
\normalsize
We refer the reader to Appendix~\ref{app:upperbd} for a proof that $\ell^\avg_\preck(\cdot)$ is an upper bounding surrogate. It is notable that for $k = n_+$ (i.e. for the PRBEP measure), the surrogate $\ell^\avg_\preck(\cdot)$ recovers Joachims' original surrogate $\ell^\struct_\preck(\cdot)$. To establish conditional consistency of this surrogate, consider the following notion of margin:
\begin{defn}[$(k,\gamma)$-margin]
\label{defn:k-margin}
A set of $n$ labeled data points satisfies the $(k,\gamma)$-margin condition if for some scoring function $s$, we have, for all sets $S_+ \subseteq \vecX_+$ of size $n_+ - k + 1$,
\[
\frac{1}{n_+ - k + 1}\sum_{i \in S_+} s_i - \max_{j : \y_j = 0} s_j \geq \gamma.
\]
Moreover, we say that the function $s$ \emph{realizes} this margin. We abbreviate the $(k,1)$-margin condition as simply the $k$-margin condition.
\end{defn}
We can now establish the consistency of $\ell^\avg_\preck(\cdot)$ under the $k$-margin condition. See Appendix~\ref{app:proof-clm-consistent} for a proof.
\begin{clm}
\label{clm:consistent}
For any scoring function $s$ that realizes the $k$-margin over a dataset, $\ell^\avg_\preck(s) = \preck(s) = 0$.
\end{clm}
We note that the $(k,\gamma)$-margin condition is strictly weaker than the \emph{strong} $\gamma$-margin condition (Definition~\ref{defn:strong-k-margin}) since it still allows a non negligible fraction of the positive points to be assigned a lower score than those assigned to negatives. On the other hand, the $(k,\gamma)$-margin condition is strictly stronger than the {\em weak} $(k,\gamma)$-margin condition (Definition~\ref{defn:weak-k-margin}). The weak $k$-margin condition only requires one set of $k$-positives to be separated from the negatives, whereas the above margin condition at least requires the average of {\em all} positives to be separated from the negatives.

\begin{figure}[t!]
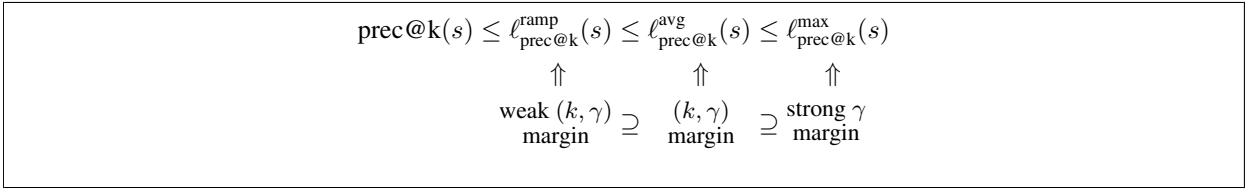

\begin{mdframed}
\vspace*{-3ex}
\begin{alignat*}{2}
	\preck(s) \leq \ell^\ramp_\preck(s) &\leq \ell^\avg_\preck(s) &&\leq \ell^\mx_\preck(s)\\
	\Uparrow \ \ \ \ \ \ \ & \ \ \ \ \ \ \ \ \ \ \ \Uparrow && \ \ \ \ \ \ \ \ \ \ \Uparrow\\
	\substack{\text{\small weak $(k,\gamma)$}\\\text{\small margin}} &\supseteq \ \ \ \substack{\text{\small $(k,\gamma)$}\\\text{\small margin}} &&\supseteq \substack{\text{\small strong $\gamma$}\\\text{\small margin}}
\end{alignat*}
\end{mdframed}
\caption{A hierarchy among the three surrogates for \preck and the corresponding margin conditions for conditional consistency.}%
\label{fig:surrogate-hierarchy}%
\end{figure}

As Figure~\ref{fig:surrogate-hierarchy} demonstrates, the three surrogates presented above, as well as their corresponding margin conditions, fall in a neat hierarchy. We will now use these surrogates to formulate two perceptron algorithms with mistake bounds with respect to these margin conditions.\


%% file: method.tex
\section{Perceptron \& SGD Algorithms for \preck}
\label{sec:perceptron}
We now present perceptron-style algorithms for maximizing the \preck performance measure in bipartite ranking settings. Our algorithms work with a stream of binary labeled points and process them in \emph{mini-batches} of a predetermined size $b$. Mini-batch methods have recently gained popularity and have been used to optimize ranking loss functions such as $\ell^\struct_\preck(\cdot)$ as well \cite{KarN014}. It is useful to note that the requirement for mini-batches goes away in ranking and multi-label classification settings, for our algorithms can be applied to individual data points in those settings (e.g. individual queries in ranking settings).

At every time instant $t$, our algorithms receive a batch of $b$ points $\vecX_t = \bs{\x_t^1,\ldots,\x_t^b}$ and rank these points using the existing model. Let $\Delta_t$ denote the \preck loss (equation~\ref{eq:preck-def}) at time $t$. If $\Delta_t = 0$ i.e. all top $k$ ranks are occupied by positive points, then the model is not updated. Otherwise, the model is updated using the false positives and negatives. For sake of simplicity, we will only look at linear models in this paper. Depending on the kind of updates we make, we get two variants of the perceptron rule for \preck.

Our first algorithm, \perckavg, updates the model using a combination of all the false positives and negatives (see Algorithm~\ref{algo:ptronATk}). The effect of the update is a very natural one -- it explicitly boosts the scores of the positive points that failed to reach the top ranks, and attenuates the scores of the negative points that got very high scores. It is interesting to note that in the limiting case of $k = 1$ and unit batch length (i.e. $b=1$), the \perckavg update reduces to that of the standard perceptron algorithm \cite{rosenblatt58a,Minsky88} for the choice $\hy_t = \text{sign}(s_t)$. Thus, our algorithm can be seen as a natural extension of the classical perceptron algorithm.

\begin{algorithm}[t]
	\caption{\small {\sc \perckavg}}
	\label{algo:ptronATk}
	\begin{algorithmic}[1]
		\small{
			\REQUIRE{Batch length $b$}
			\STATE $\w^0 \< \vz, t \< 0$
			\WHILE{stream not exhausted}
				\STATE $t \< t+1$
				\STATE Receive $b$ data points $\vecX_t = \bs{\x_t^1,\ldots,\x_t^b}$, $\y_t \in \{0,1\}^b$
				\STATE Calculate $s_t = \w^{t-1}\vecX_t$ and let $\hy_t = \y^{(s_t,k)}$
				\STATE $\Delta_t \< \Delta(\y_t,\hy_t)$
				\IF{$\Delta_t = 0$}
					\STATE $\w^t \< \w^{t-1}$
				\ELSE
					\STATE $D_t \< \frac{\Delta_t}{\norm{\y_t}_1 - K(\y_t,\hy_t)}$
					\STATE $\w^t \< \w^{t-1} - \sum_{i \in [b]} (1-\y_i)\hy_i\cdot\x_t^i$ \COMMENT{false positives}
					\STATE $\w^t \< \w^t + D_t\cdot\sum_{i \in [b]} (1-\hy_i)\y_i\cdot\x_t^i$ \COMMENT{false negatives}
				\ENDIF
			\ENDWHILE
			\STATE \textbf{return} {$\w^t$}
		}
	\end{algorithmic}
\end{algorithm}

\begin{algorithm}[t]
	\caption{\small {\sc \perckmax}}
	\label{algo:ptronATk-fast}
	\begin{algorithmic}[1]
		\small{
			\makeatletter\setcounter{ALC@line}{9}\makeatother
			\STATE \qquad$S_t \< \FN(s,\Delta_t)$
			\STATE \qquad$\w^t \< \w^{t-1} - \sum_{i \in [b]} (1-\y_i)\hy_i\cdot\x_t^i$ \COMMENT{false positives}
			\STATE \qquad$\w^t \< \w^t + \sum_{i \in S_t} \x_t^i$ \COMMENT{top ranked false negatives}
		}
	\end{algorithmic}
\end{algorithm}

\begin{algorithm}[t]
	\caption{\small {\sc \sgdk}}
	\label{algo:stoc-grad-preck}
	\begin{algorithmic}[1]
		\small{
			\REQUIRE{Batch length $b$, step lengths $\eta_t$, feasible set $\W$}
			\ENSURE{A model $\bar\w \in \W$}
			\STATE $\w^0 \< \vz, t \< 0$
			\WHILE{stream not exhausted}
				\STATE $t \< t+1$
				\STATE Receive $b$ data points $\vecX_t = \bs{\x_t^1,\ldots,\x_t^b}$, $\y_t \in \{0,1\}^b$
				\STATE Set $\g_t \in \partial_\w\ell^\avg_\preck(\w_{t-1};\vecX_t,\y_t)$\COMMENT{See Algorithm~\ref{algo:grad-calc}}
				\STATE $\w_t \< \Pi_\W\bs{\w_{t-1} - \eta_t\cdot\g_t}$\COMMENT{project onto set $\W$}
			\ENDWHILE
			\STATE \textbf{return} {$\bar\w = \frac{1}{t}\sum_{\tau=1}^t\w_\tau$}
		}
	\end{algorithmic}
\end{algorithm}

\begin{algorithm}[t]
	\caption{\small Subgradient calculation for $\ell^\avg_\preck(\cdot)$}
	\label{algo:grad-calc}
	\begin{algorithmic}[1]
		\small{
			\REQUIRE{A model $\w_{\text{in}}$, $n$ data points $\vecX, \y$, parameter $k$}
			\ENSURE{A subgradient $\g \in \partial_\w\ell^\avg_\preck(\w_{\text{in}};\vecX,\y)$}
			\STATE Sort pos. and neg. points separately in dec. order of scores assigned by $\w_{\text{in}}$ i.e. $s_1^+ \geq \ldots \geq s_{n_+}^+$ and $s_1^- \geq \ldots \geq s_{n_-}^-$
			\FOR{$k' = 0 \> k$}
				\STATE $D_{k'} \< \frac{k - k'}{n_+ - k'}$
				\STATE $\Delta_{k'} \< k -k' - D_{k'}\sum_{i=k'+1}^{n_+}s^+_i + \sum_{i=1}^{k-k'}s^-_i$
				\STATE $\g_{k'} \< \sum_{i=1}^{k-k'}\x^-_i - D_{k'}\sum_{i=k'+1}^{n_+}\x^+_i$
			\ENDFOR
			\STATE $k^\ast \< \arg\max_{k'} \Delta_{k'}$
			\STATE \textbf{return} {$\g_{k^\ast}$}
		}
	\end{algorithmic}
\end{algorithm}

The next lemma establishes that, similar to the classical perceptron \cite{novikoff1962convergence}, \perckavg also enjoys a mistake bound. Our mistake bound is stated in the most general agnostic setting with the hinge loss function replaced with our surrogate $\ell^\avg_\preck(s)$. All proofs in this section are deferred to Appendix~\ref{app:prf_perc}.
\begin{thm}
\label{thm:ptron-mistake}
Suppose $\norm{\x_t^i} \leq R$ for all $t,i$. Let $\Delta^C_T = \sum_{t=1}^T\Delta_t$ be the cumulative mistake value observed when Algorithm~\ref{algo:ptronATk} is executed for $T$ batches. Also, for any $\w$, let $\hat\L^\avg_T(\w) = \sum_{t=1}^T \ell^\avg_\preck(\w;\vecX_t,\y_t)$. Then we have
\[
\Delta^C_T \leq \min_{\w}\ {\br{\norm{\w}\cdot R\cdot\sqrt{4k} + \sqrt{\hat\L^\avg_T(\w)}}^2}.
\]
\end{thm}
Similar to the classical perceptron mistake bound \cite{novikoff1962convergence}, the above bound can also be reduced to a simpler convergence bound in \emph{separable settings}.
\begin{cor}
Suppose a unit norm $\w^\ast$ exists such that the scoring function $s: \x \mapsto \x^\top\w^\ast$ realizes the $(k,\gamma)$-margin condition for all the batches, then Algorithm~\ref{algo:ptronATk} guarantees the mistake bound: $\ \ \Delta^C_T \leq \frac{4kR^2}{\gamma^2}.$
\end{cor}
The above result assures that, as datasets become ``easier'' in the sense that their $(k,\gamma)$-margin becomes larger, \perckavg will converge to an optimal hyperplane at a faster rate. It is important to note there that the $(k,\gamma)$-margin condition is strictly weaker than the standard classification margin condition. Hence for several datasets, \perckavg might be able to find a perfect ranking while at the same time, it might be impossible for standard binary classification techniques to find any reasonable classifier in poly-time \cite{GuruswamiR09}. 

We note that \perckavg performs updates with all the false negatives in the mini-batches. This raises the question as to whether sparser updates are possible as such updates would be slightly faster as well as, in high dimensional settings, ensure that the model is sparser. To this end we design the \perckmax algorithm (Algorithm~\ref{algo:ptronATk-fast}). \perckmax differs from \perckavg in that it performs updates using only a few of the {\em top ranked} false negatives. More specifically, for any scoring function $s$ and $m > 0$, define:
\[
\FN(s,m) = \mathop{\arg\max}_{{S \subset \vecX_t^+, |S|=m}}\ \sum_{i \in S}\br{1-\y^{(s,k)}_i}\y_is_i
\]
as the set of the $m$ top ranked false negatives. \perckmax makes updates only for false positives in the set $\FN(s,\Delta_t)$. Note that $\Delta_t$ can significantly smaller than the total number of false negatives if $k \ll n_+$. \perckmax also enjoys a mistake bound but with respect to the $\ell^{\text{max}}_{\preck}(\cdot)$ surrogate.
\begin{thm}
\label{thm:ptron-fast-mistake}
Suppose $\norm{\x_t^i} \leq R$ for all $t,i$. Let $\Delta^C_T = \sum_{t=1}^T\Delta_t$ be the cumulative observed mistake value when Algorithm~\ref{algo:ptronATk-fast} is executed for $T$ batches. Also, for any $\w$, let $\hat\L^{\mx}_T(\w) = \sum_{t=1}^T \ell^{\mx}_\preck(\w;\vecX_t,\y_t)$. Then we have
\[
\Delta^C_T \leq \min_{\w}\ {\br{\norm{\w}\cdot R\cdot\sqrt{4k} + \sqrt{\hat\L^{\mx}_T(\w)}}^2}.
\]
\end{thm}
Similar to \perckavg, we can give a simplified mistake bound in situations where the separability condition specified by Definition~\ref{defn:strong-k-margin} is satisfied.
\begin{cor}
Suppose a unit norm $\w^\ast$ exists such that the scoring function $s: \x \mapsto \x^\top\w^\ast$ realizes the \emph{strong} $\gamma$-margin condition for all the batches, then Algorithm~\ref{algo:ptronATk-fast} guarantees the mistake bound: $\ \ \Delta^C_T \leq \frac{4kR^2}{\gamma^2}.$
\end{cor}
As the \emph{strong} $\gamma$-margin condition is exactly the same as the standard notion of margin for binary classification, the above bound is no stronger than the one for the classical perceptron. However, in practice, we observe that \perckmax at times outperforms even \perckavg, even though the latter has a tighter mistake bound. This suggests that our analysis of \perckmax might not be optimal and fails to exploit latent structures that might be present in the data. 

\textbf{Stochastic Gradient Descent for Optimizing \preck.}\\
We now extend our algorithmic repertoire to include a stochastic gradient descent (SGD) algorithm for the \preck performance measure. SGD methods are known to be very successful at optimizing large-scale empirical risk minimization (ERM) problems as they require only a few passes over the data to achieve optimal statistical accuracy.

However, SGD methods typically require access to cheap gradient estimates which are difficult to obtain for non-additive performance measures such as \preck. This has been noticed before by several previous works \cite{KarN014, NarasimhanKJ15} who propose to use mini-batch methods to overcome this problem \cite{KarN014}. By combining the $\ell^\avg_\preck(\cdot)$ surrogate with mini-batch-style processing, we design \sgdk (Algorithm~\ref{algo:stoc-grad-preck}), a scalable SGD algorithm for optimizing \preck. The algorithm uses mini-batches to update the current model using gradient descent steps. The subgradient calculation for this surrogate turns out to be non-trivial and is detailed in Algorithm~\ref{algo:grad-calc}.

The task of analyzing this algorithm is made non-trivial by the fact that the gradient estimates available to \sgdk via Algorithm~\ref{algo:grad-calc} are far from being unbiased. The luxury of having unbiased gradient estimates is crucially exploited by standard SGD analyses but unfortunately, unavailable to us. To overcome this hurdle, we propose a uniform convergence based proof that, in some sense, bounds the bias in the gradient estimates.

In the following section, we present this, and many other generalization and online-to-batch conversion bounds with applications to our perceptron and SGD algorithms.
 

%% file: genbound.tex
\section{Generalization Bounds}\label{sec:genbound}
In this section, we discuss novel uniform convergence (UC) bounds for our proposed surrogates. We will use these UC bounds along with the mistake bounds in Theorems~\ref{thm:ptron-mistake}~and~\ref{thm:ptron-fast-mistake} to prove two key results -- 1) online-to-batch conversion bounds for the \perckavg and \perckmax algorithms and, 2) a convergence guarantee for the \sgdk algorithm.

To better present our generalization and convergence bounds, we use normalized versions of \preck and the surrogates. To do so we write $k = \kappa\cdot n_+$ for some $\kappa \in (0,1]$ and define, for any scoring function $s$, its \preckappa loss as:
\[
\preckappa(s;\ \z_1,\ldots,\z_n) = \frac{1}{\kappa n_+}\Delta(\y,\y^{(s,\kappa n_+)}).
\]
We will also normalize the surrogate functions $\ell^\ramp_\preckappa(\cdot)$, $\ell^\mx_\preckappa(\cdot)$, and $\ell^\avg_\preckappa(\cdot)$ by dividing by $k =\kappa\cdot n_+$.
\begin{defn}[Uniform Convergence]
A performance measure $\Psi : \W \times (\X \times \bc{0,1})^n \mapsto \R_+$ exhibits uniform convergence with respect to a set of predictors $\W$ if for some $\alpha(b,\delta) = \text{poly}\br{\frac{1}{b},\log\frac{1}{\delta}}$, for a sample $\hz_1,\ldots,\hz_b$ of size $b$ chosen i.i.d. (or uniformly without replacement) from an arbitrary population $\z_1,\ldots,\z_n$, we have w.p. $1 - \delta$,
\[
\sup_{\w\in\W}\ \abs{\Psi(\w;\ \z_1,\ldots,\z_n) - \Psi(\w;\ \hz_1,\ldots,\hz_b)} \leq \alpha(b,\delta)
\]
\end{defn}
We now state our UC bounds for \preckappa and its surrogates. We refer the reader to Appendix~\ref{app:uc-bounds-surrogates} for proofs.
\begin{thm}
\label{thm:uc-preck-surrogates}
The loss function $\preckappa(\cdot)$, as well as the surrogates $\ell^\ramp_\preckappa(\cdot)$, $\ell^\avg_\preckappa(\cdot)$ and $\ell^\mx_\preckappa(\cdot)$, all exhibit uniform convergence at the rate $\alpha(b,\delta) = \O{\sqrt{\frac{1}{b}\log\frac{1}{\delta}}}$.
\end{thm}
Recently, \cite{KarN014} also established a similar result for the $\ell^\struct_\preck(\cdot)$ surrogate. However, a very different proof technique is required to establish similar results for $\ell^\mx_\preckappa(\cdot)$ and $\ell^\avg_\preckappa(\cdot)$, partly necessitated by the terms in these surrogates which depend, in a complicated manner, on the positives predicted by the candidate labeling $\hy$. Nevertheless, the above results allow us to establish strong online-to-batch conversion bounds for \perckavg and \perckmax, as well as convergence rates for the \sgdk method. In the following we shall assume that our data streams are composed of points chosen i.i.d. (or u.w.r.) from some fixed population $\ZZ$.
\newcommand{\sgdmax}{SGD@k-max}
\begin{figure*}[t]
\centering\hspace*{-20pt}
\subfigure[PPI]{
\includegraphics[scale=0.51]{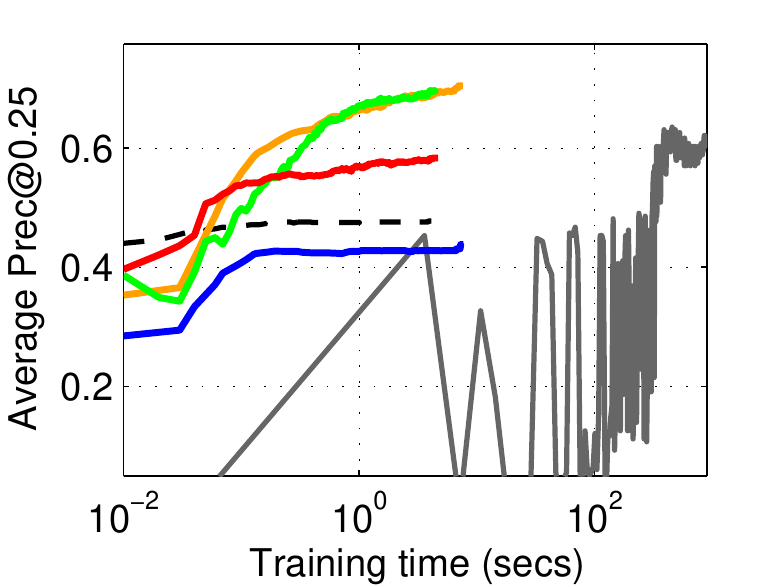}
\label{subfig:ppi}
}\hspace*{-10pt}
\subfigure[Letter]{
\includegraphics[scale=0.51]{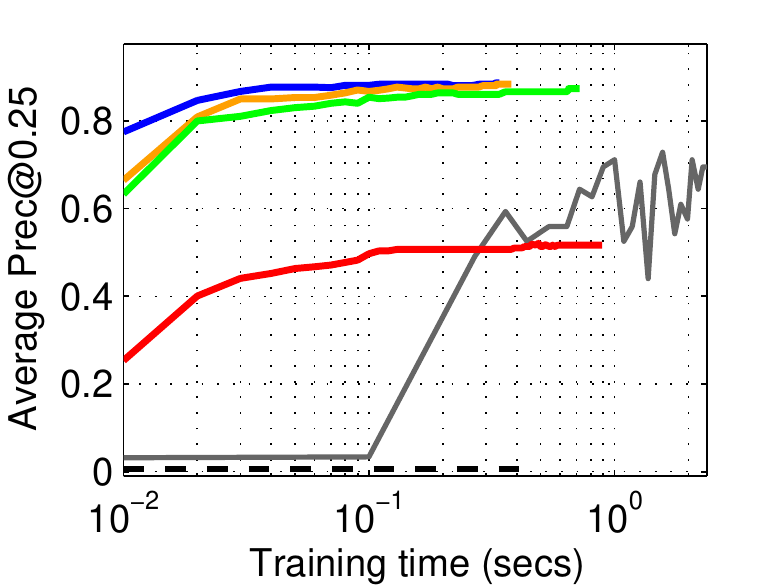}
\label{subfig:letter}
}\hspace*{-10pt}
\subfigure[Adult]{
\includegraphics[scale=0.51]{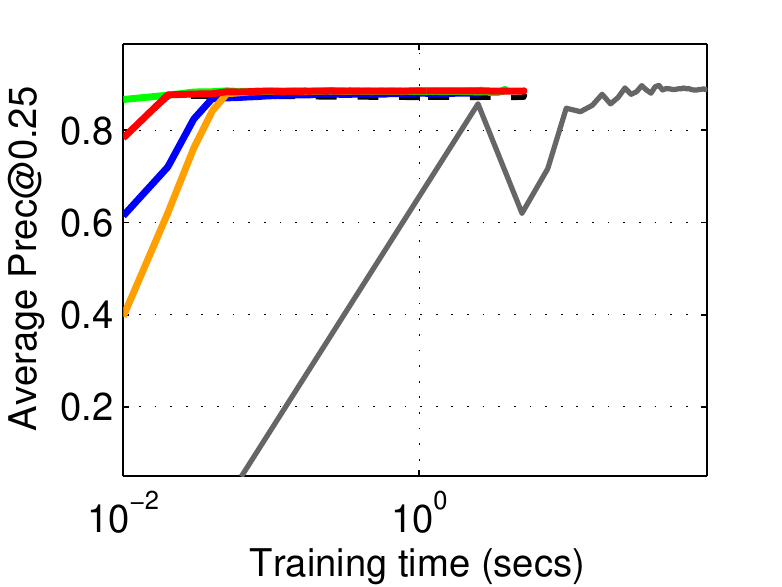}
\label{subfig:a9a}
}\hspace*{-10pt}
\subfigure[IJCNN]{
\includegraphics[scale=0.51]{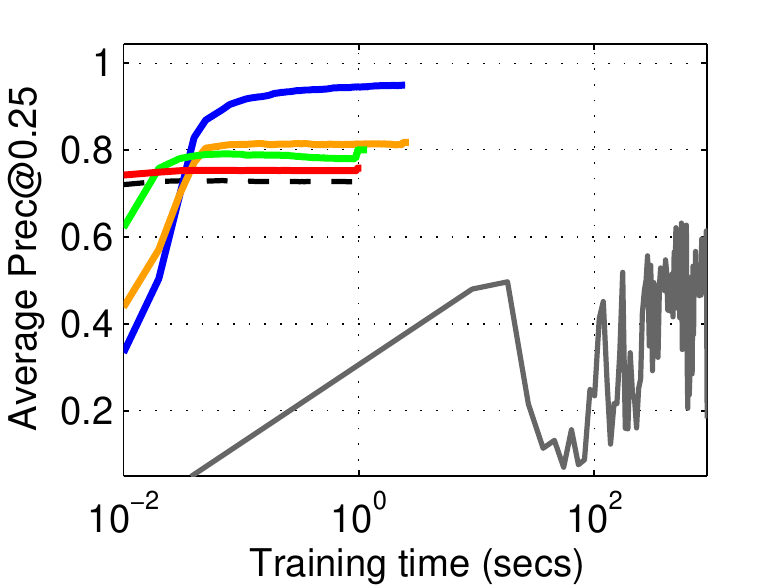}
\label{subfig:ijcnn1}
}\hspace*{-10pt}
\subfigure{
\includegraphics[scale=0.15]{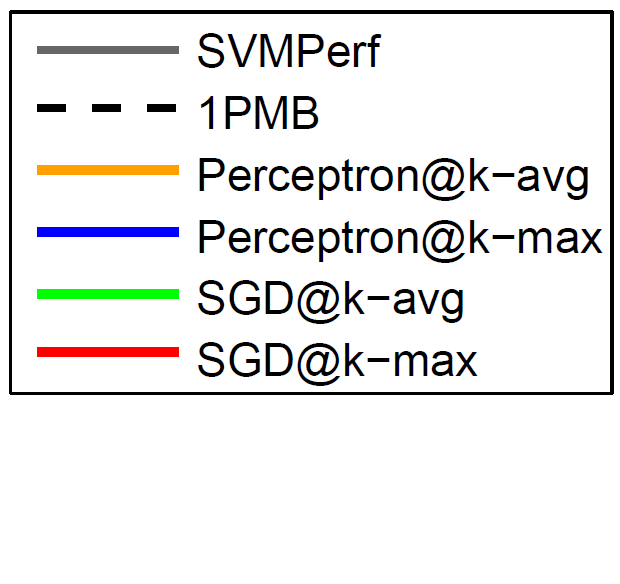}
}\vspace*{-5pt}
\caption{A comparison of the proposed perceptron and SGD based methods with baseline methods (SVMPerf and \spmb) on prec@0.25 maximization tasks. \perckavg and \sgdk (both based on $\ell^\avg_\preck(\cdot)$) are the most consistent methods across tasks.}
\label{fig:training-time}
\end{figure*}

\begin{figure*}[t]
\centering\hspace*{-15pt}
\subfigure[KDD08]{
\includegraphics[scale=0.4]{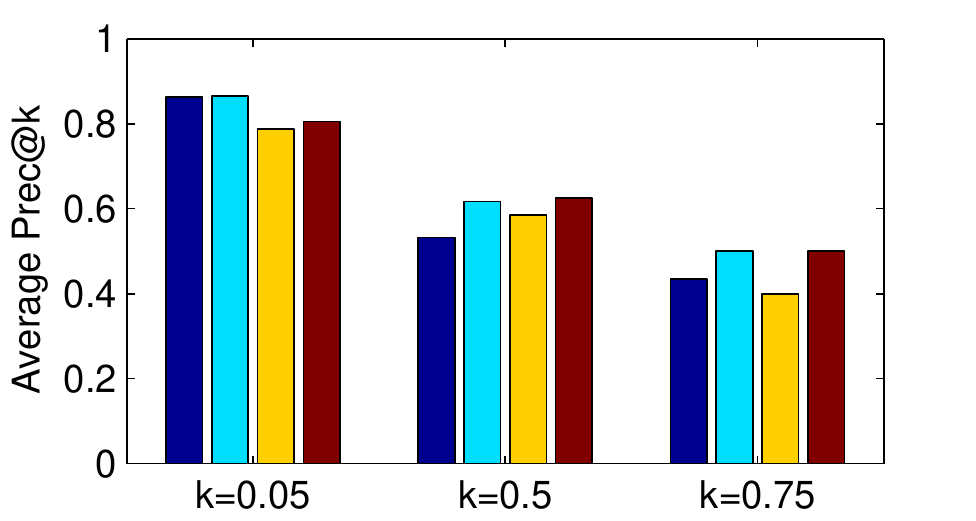}
\label{subfig:bar-kdd08}
}\hspace*{-15pt}
\subfigure[PPI]{
\includegraphics[scale=0.4]{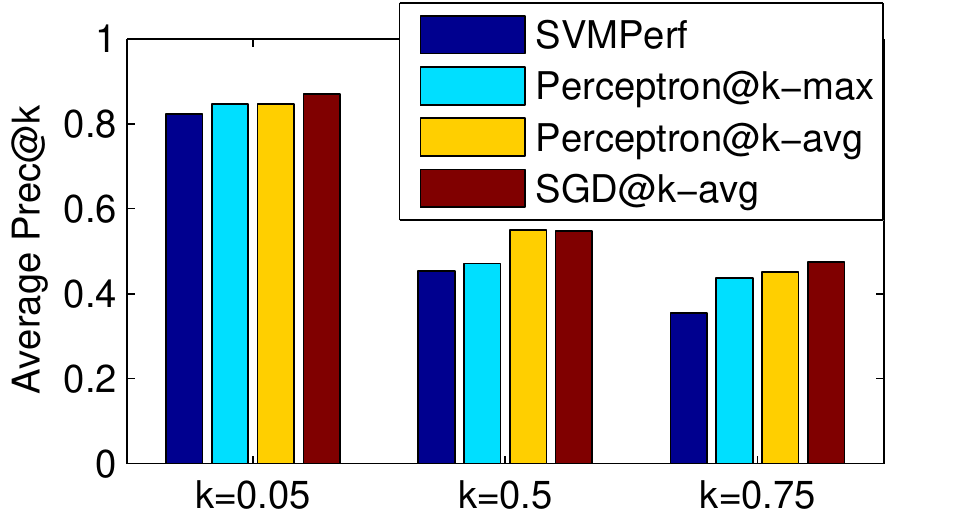}
\label{subfig:bar-ppi}
}\hspace*{-10pt}
\subfigure[KDD08]{
\includegraphics[scale=0.5]{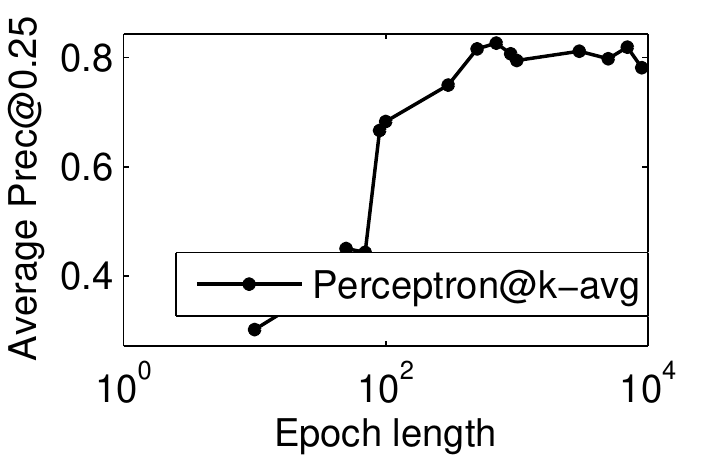}
\label{subfig:ppi-pavg}
}\hspace*{-10pt}
\subfigure[KDD08]{
\includegraphics[scale=0.5]{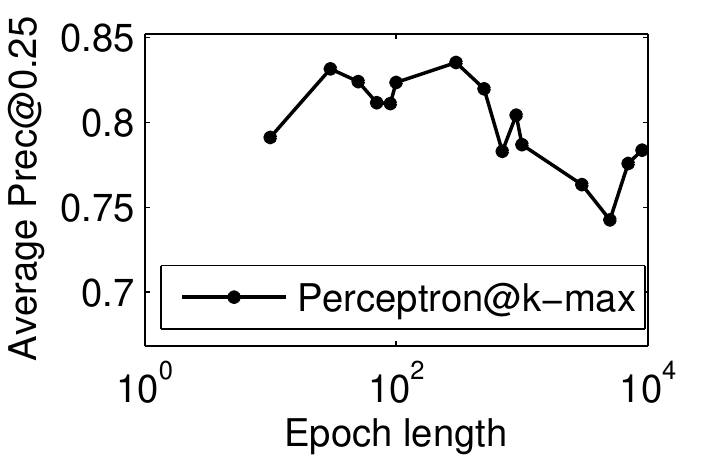}
\label{subfig:ppi-pmax}
}\hspace*{-10pt}
\subfigure[KDD08]{
\includegraphics[scale=0.5]{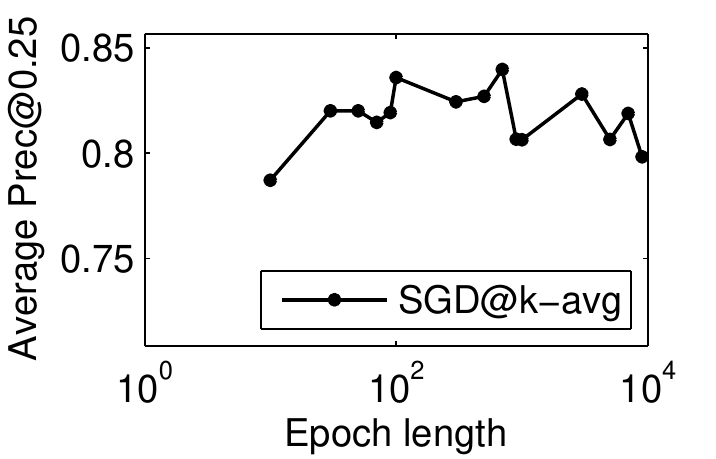}
\label{subfig:ppi-savg}
}\vspace*{-5pt}
\caption{(a), (b): A comparison of different methods on optimizing \preckappa for different values of $\kappa$. (c), (d), (e): The performance of the proposed perceptron and SGD methods on prec@0.25 maximization tasks with varying batch lengths $b$.}
\label{fig:beta}
\end{figure*}
\begin{thm}
Suppose an algorithm, when fed a random stream of data points, in $T$ batches of length $b$ each, generates an ensemble of models $\w_1,\ldots,\w_T$ which together suffer a cumulative mistake value of $\Delta_T^C$. Then, with probability at least $1 - \delta$, we have
\[
\frac{1}{T}\sum_{t=1}^T\preckappa(\w^t;\ZZ) \leq \frac{\Delta_T^C}{bT} + \O{\sqrt{\frac{1}{b}\log\frac{T}{\delta}}}.
\]
\end{thm}
The proof of this theorem follows from Theorem~\ref{thm:uc-preck-surrogates} which guarantees that w.p. $1 - \delta$, $\preckappa(\w^t;\ZZ) \leq \Delta_t/b + \O{\sqrt{\frac{1}{b}\log\frac{1}{\delta}}}$ for all $t$. Combining this with the mistake bound from Theorem~\ref{thm:ptron-mistake} ensures the following generalization guarantee for the ensemble generated by Algorithm~\ref{algo:ptronATk}.
\begin{cor}
Let $\w^1,\ldots,\w^T$ be the ensemble of classifiers returned by the \perckavg algorithm on a random stream of data points and batch length $b$. Then, with probability at least $1 - \delta$, for any $\w^\ast$ we have
{\[
\frac{1}{T}\sum_{t=1}^T\preckappa(\w^t;\ZZ) \leq \br{\sqrt{\ell^\avg_\preckappa(\w^\ast;\ZZ)} + C}^2,
\]}
where {$C = \O{\norm{\w^\ast}R\sqrt{\frac{1}{T}} + \sqrt[4]{\frac{1}{b}\log\frac{T}{\delta}}}$}.
\end{cor}
A similar statement holds for the {\small\sc Perceptron@k-max} algorithm with respect to the $\ell^\mx_\preckappa(\cdot)$ surrogate as well. Using the results from Theorem~\ref{thm:uc-preck-surrogates}, we can also establish the convergence rate of the \sgdk algorithm.
\begin{thm}
\label{thm:conv-ssgd@k}
Let $\bar\w$ be the model returned by Algorithm~\ref{algo:stoc-grad-preck} when executed on a stream with $T$ batches of length $b$. Then with probability at least $1 -\delta$, for any $\w^\ast \in \W$, we have
\vskip-3ex
\small
\[
\ell^\avg_\preckappa(\bar\w;\ZZ) \leq \ell^\avg_\preckappa(\w^\ast;\ZZ) + \O{\sqrt{\frac{1}{b}\log\frac{T}{\delta}}}+ \O{\sqrt\frac{1}{T}}
\]
\normalsize
\end{thm}
The proof of this Theorem can be found in Appendix~\ref{app:thm-conv-ssgd@k-proof}.


%% file: exps.tex
\section{Experiments}
\label{sec:exps}
We shall now evaluate our methods on several benchmark datasets for binary classification problems with a rare-class.

\textbf{Datasets}: We evaluated our methods on 7 publicly available benchmark datasets: a) PPI, b) KDD Cup 2008, c) Letter, d) Adult, e) IJCNN, f) Covertype, and g) Cod-RNA. All datasets exhibit moderate to severe label imbalance with the KDD Cup dataset having just 0.61\% positives.

\textbf{Methods}: We compared both perceptron algorithms, \sgdk, as well as an SGD solver for the $\ell^\mx_\preck(\cdot)$ surrogate, with the cutting plane-based SVMPerf solver of \cite{Joachims05}. We also compare against stochastic \spmb solver of \cite{KarN014}. The perceptron and SGD methods were given a maximum of 25 passes over the data with a batch length of $500$. All methods were implemented in C. We used 70\% of the data for training and the rest for testing. All results are averaged over $5$ random train-test splits.

Our experiments reveal three interesting insights into the problem of \preck maximization -- 1) using tighter surrogates for optimization routines is indeed beneficial, 2) the presence of a stochastic solver cannot always compensate for the use of a suboptimal surrogate, and 3) mini-batch techniques, applied with perceptron or SGD-style methods, can offer rapid convergence to accurate models.

We first timed all the methods on \preckappa maximization tasks for $\kappa = 0.25$ on various datasets (see Figure~\ref{fig:training-time}). Of all the methods, the cutting plane method (SVMPerf) was found to be the most expensive computationally. On the other hand, the perceptron and stochastic gradient methods, which make frequent but cheap updates, were much faster at identifying accurate solutions.

We also observed that \perckavg and \sgdk, which are based on the tight $\ell^\avg_\preck(\cdot)$ surrogate, were the most consistent at converging to accurate solutions whereas \perckmax and \sgdkmax, which are based on the loose $\ell^\mx_\preck(\cdot)$ surrogate, showed large deviations in performance across tasks. Also, \spmb and SVMPerf, which are based on the non upper-bounding $\ell^\struct_\preck(\cdot)$ surrogate, were frequently found to converge to suboptimal solutions.

The effect of working with a tight surrogate is also clear from Figure~\ref{fig:beta} (a), (b) where the algorithms working with our novel surrogates were found to consistently outperform the SVMPerf method which works with the $\ell^\struct_\preck(\cdot)$ surrogate. For these experiments, SVMPerf was allowed a runtime of up to $50\times$ of what was given to our methods after which it was terminated.

Finally, to establish the stability of our algorithms, we ran, both the perceptron, as well as the SGD algorithms with varying batch lengths (see Figure~\ref{fig:beta} (c)-(e)).  We found the algorithms to be relatively stable to the setting of the batch length. To put things in perspective, all methods registered a relative variation of less than 5\% in accuracies across batch lengths spanning an order of magnitude or more. We present additional experimental results in Appendix~\ref{app:exps}.


%% file: app.tex
\input{structsvm}

\section{Proofs of Claims from Section~\ref{sec:new-surrogate}}\label{app:prf_surrogate}
\subsection{Proof of Claim~\ref{clm:upper-bd-tight}}
\begin{repclm}{clm:upper-bd-tight}
For any $k \leq n_+$ and scoring function $s$, we have
\[
\ell^\ramp_\preck(s) \geq \preck(s).
\]
Moreover, if for some scoring function $s$, we have $\ell^\ramp_\preck(s) \leq \xi$, then there necessarily exists a set $S \subset [n]$ of size at most $k$ such that for all $\norm{\hy} = k$, we have
\[
\sum_{i \in S}s_i \geq \sum_{i=1}^n\hy_is_i + \Delta(\y,\hy) - \xi.
\]
\end{repclm}
\begin{proof}
Let $\hy = \y^{(s,k)}$ so that we have $\Delta(\y,\hy) = \preck(s)$. Then we have
\begin{align*}
\ell^\ramp_\preck(s) &= \max_{\norm{\hy}_1=k}\bc{\Delta(\y,\hy) + \sum_{i=1}^n\hy_is_i} - {\max_{\substack{\norm{\ty}_1 = k\\K(\y,\tilde\y) = k}}\sum_{i=1}^n\tilde\y_is_i}\\
					&\geq {\Delta(\y,\hy) + \sum_{i=1}^n\hy_is_i} - {\max_{\substack{\norm{\ty}_1 = k\\K(\y,\tilde\y) = k}}\sum_{i=1}^n\ty_is_i}\\
					&= {\Delta(\y,\hy) + \max_{\norm{\ty}_1=k}\sum_{i=1}^n\ty_is_i} - {\max_{\substack{\norm{\ty}_1 = k\\K(\y,\tilde\y) = k}}\sum_{i=1}^n\ty_is_i}\\
					&\geq \Delta(\y,\hy),
\end{align*}
where the third step follows from the definition of $\hy$. This proves the first claim. For the second claim, suppose for some scoring function $s$, we have $\ell^\ramp_\preck(s) \leq \xi$. Then if we consider $S^\ast$ to be the set of $k$-highest ranked positive points, then we have
\[
\sum_{i \in S^\ast}s_i = \max_{\substack{\norm{\ty}_1 = k\\K(\y,\tilde\y) = k}}\sum_{i=1}^n\tilde\y_is_i \geq \max_{\norm{\hy}_1=k}\bc{\Delta(\y,\hy) + \sum_{i=1}^n\hy_is_i} - \xi \geq \sum_{i=1}^n\hy_is_i + \Delta(\y,\hy) - \xi,
\]
which proves the claim.
\end{proof}
\subsection{Proof of Claim~\ref{clm:consistent-weak}}
\begin{repclm}{clm:consistent-weak}
For any scoring function $s$ that realizes the \emph{weak} $k$-margin over a dataset we have,
\[
\ell^\ramp_\preck(s) = \preck(s) = 0.
\]
\end{repclm}
\begin{proof}
Consider a scoring function $s$ that satisfies the weak $k$-margin condition and any $\hy$ such that $\norm{\hy}_1 = k$. Based on the \preck accuracy of $\hy$, we have the following two cases

\textbf{Case 1} ($K(\y,\hy) = k$): In this case we have
\[
\Delta(\y,\hy) + \sum_{i=1}^n\hy_is_i - {\max_{\substack{\norm{\ty}_1 = k\\K(\y,\tilde\y) = k}}\sum_{i=1}^n\tilde\y_is_i} = 0 + \sum_{i=1}^n\hy_is_i - \max_{\substack{\norm{\ty}_1 = k\\K(\y,\tilde\y) = k}}\sum_{i=1}^n\tilde\y_is_i \leq 0,
\]
where the first step follows since $K(\y,\hy) = k$ and the second step follows since $\norm{\hy}_1 = k$, as well as $K(\y,\hy) = k$.

\textbf{Case 2} ($K(\y,\hy) = k' < k$): In this case let $S^\ast$ be the set of $k$ top ranked positive points according to the scoring function $s$. Also let $S^\ast_1$ be the set of $k' (= K(\y,\hy))$ top ranked positives and let $S^\ast_2 = S^\ast\backslash S^\ast_1$. Then we have
\begin{align*}
\Delta(\y,\hy) + \sum_{i=1}^n\hy_is_i - {\max_{\substack{\norm{\ty}_1 = k\\K(\y,\tilde\y) = k}}\sum_{i=1}^n\tilde\y_is_i} &= \Delta(\y,\hy) + \underbrace{\sum_{i=1}^n\hy_i\y_is_i}_{(A)} + \sum_{i=1}^n\hy_i(1-\y_i)s_i - {\max_{\substack{\norm{\ty}_1 = k\\K(\y,\tilde\y) = k}}\sum_{i=1}^n\tilde\y_is_i}\\
							&\leq \Delta(\y,\hy) + \sum_{i\in S^\ast_1}s_i + \underbrace{\sum_{i=1}^n\hy_i(1-\y_i)s_i}_{(B)} - {\max_{\substack{\norm{\ty}_1 = k\\K(\y,\tilde\y) = k}}\sum_{i=1}^n\tilde\y_is_i}\\
							&\leq \Delta(\y,\hy) + \sum_{i\in S^\ast_1}s_i + \sum_{i\in S^\ast_2}s_i - (k-k') - {\max_{\substack{\norm{\ty}_1 = k\\K(\y,\tilde\y) = k}}\sum_{i=1}^n\tilde\y_is_i}\\
							&= k-k' + \sum_{i\in S^\ast}s_i - (k-k') - {\max_{\substack{\norm{\ty}_1 = k\\K(\y,\tilde\y) = k}}\sum_{i=1}^n\tilde\y_is_i}\\
							&= 0,
\end{align*}
where the second step follows since the term $(A)$ consists of $k'$ true positives the third step follows since the term $(B)$ contains $k-k'$ false positives i.e. negatives and the $k$-margin condition, the fourth step follows since $\Delta(\y,\hy) = k - K(\y,\hy)$ and the fifth step follows since by the definition of the set $S^\ast$, we have
\[
\sum_{i\in S^\ast}s_i = \max_{\substack{\norm{\ty}_1 = k\\K(\y,\tilde\y) = k}}\sum_{i=1}^n\tilde\y_is_i.
\]
In both cases, we have shown the surrogate to be non-positive. Since the performance measure \preck cannot take negative values, this, along with the upper bounding property implies that $\preck(s) = 0$ as well. This finishes the proof.
\end{proof}

\subsection{A Useful Supplementary Lemma}
\begin{lem}\label{lem:rank-ineq}
Given a set of $n$ real numbers $x_1 \ldots x_n$ and any two integers $k \leq k' \leq n$, we have
\[
\min_{|S| = k}\frac{1}{k}\sum_{i \in S}x_i \leq \min_{|S'| = k'}\frac{1}{k'}\sum_{j \in S'}x_j
\]
\end{lem}
\begin{proof}
The above is obviously true if $k=k'$ so we assume that $k' > k$. Without loss of generality assume that the set is ordered in ascending order i.e. $x_1 \leq x_2 \leq \ldots \leq x_n$. Thus, the above statement is equivalent to showing that
\[
\frac{1}{k}\sum_{i=1}^kx_i \leq \frac{1}{k'}\sum_{j=1}^{k'}x_j \Leftrightarrow \br{\frac{1}{k} - \frac{1}{k'}}\sum_{i=1}^kx_i \leq \frac{1}{k'}\sum_{j=k+1}^{k'}x_j \Leftrightarrow \frac{1}{k}\sum_{i=1}^kx_i \leq \frac{1}{k'-k}\sum_{j=k+1}^{k'}x_j,
\]
where the last inequality is true since $k - k' > 0$ and the left hand side is the average of numbers which are all smaller than the numbers whose average forms the right hand side. This proves the lemma.
\end{proof}

\subsection{Proof of the Upper-bounding Property for the $\ell^\avg_\preck(\cdot)$ Surrogate}
\label{app:upperbd}
\begin{clm}\label{clm:upper-bd}
For any $k \leq n_+$ and scoring function $s$, we have
\[
\ell^\avg_\preck(s) \geq \preck(s).
\]
Moreover, for linear scoring functions i.e. $s(\x_i) = \w^\top\x_i$ for $\w \in \W$, the surrogate $\ell^\avg_\preck(\w)$ is convex in $\w$.
\end{clm}
\begin{proof}
We use the fact observed before that for any scoring function, we have $\Delta(\y,\y^{(s,k)}) = \preck(s)$. We start off by showing the second part of the claim. Recall the definition of the surrogate $\ell^\avg_\preck(s)$
\[
\ell^\avg_\preck(\w) = \max_{\norm{\hy}_1=k}\ \bc{\Delta(\y,\hy) + \sum_{i=1}^n(\hy_i-\y_i)\cdot\w^\top\x_i + \frac{1}{C(\hy)}\sum_{i=1}^n(1 - \hy_i)\y_i\cdot\w^\top\x_i}
\]
For sake of simplicity, for any $\hy \in \bc{0,1}^n$, define
\[
\Delta(s,\hy) = \Delta(\y,\hy) + \sum_{i=1}^ns_i(\hy_i-\y_i) + \frac{1}{C(\hy)}\sum_{i=1}^n(1 - \hy_i)\y_i s_i.
\]
The convexity of $\ell^\avg_\preck(\w)$ follows from the observation that the inner term in the maximization is linear (hence convex) in $\w$ and the $\max$ function is convex and increasing. We now move on to prove the first part. For sake of convenience $\ty = \y^{(s,k)}$. Note that $\norm{\ty}_1 = k$ by definition. This gives us
\begin{eqnarray*}
\ell^\avg_\preck(s) &=& \max_{\norm{\hy}_1=k}\Delta(s,\hy) \geq \Delta(s,\ty)\\
				&=& \Delta(\y,\ty) + \sum_{i=1}^ns_i(\ty_i-\y_i) + \frac{1}{C(\ty)}\sum_{i=1}^n(1 - \ty_i)\y_is_i\\
				&=& \Delta(\y,\ty) + \sum_{i=1}^ns_i(\ty_i(1-\y_i)-\y_i(1-\ty_i)) + \frac{n_+-k}{n_+-K(\y,\ty)}\sum_{i=1}^n(1 - \ty_i)\y_is_i\\
				&=& \Delta(\y,\ty) + \underbrace{\sum_{i=1}^n\ty_i(1-\y_i)s_i}_{(A)} - \underbrace{\frac{k - K(\y,\ty)}{n_+-K(\y,\ty)}\sum_{i=1}^n(1 - \ty_i)\y_is_i}_{(B)}.
\end{eqnarray*}
Now define $m = \min_{\substack{\ty_i = 1\\\y_i = 0}}\ s_i$ and $M = \max_{\substack{\ty_i = 0\\\y_i = 1}}\ s_i$. This gives us
\[
(A) = \sum_{i=1}^n\ty_i(1-\y_i)s_i \geq m\sum_{i=1}^n\ty_i(1-\y_i) = \Delta(\y,\ty)\cdot m,
\]
and
\[
(B) = \frac{k - K(\y,\ty)}{n_+-K(\y,\ty)}\sum_{i=1}^n(1 - \ty_i)\y_is_i \leq \frac{k - K(\y,\ty)}{n_+-K(\y,\ty)}\sum_{i=1}^n(1 - \ty_i)\y_i M = (k - K(\y,\ty))\cdot M = \Delta(\y,\ty)\cdot M.
\]
However, by definition of $\ty = \y^{(s,k)}$, we have
\[
m \geq \min_{\ty = 1}\ s_i \geq \max_{\ty = 0}\ s_i \geq M.
\]
Thus we have
\[
\ell^\avg_\preck(s) \geq \Delta(\y,\ty) + (A) - (B) \geq \Delta(\y,\ty)(1 + m - M) \geq \Delta(\y,\ty) = \preck(s)\qedhere
\]
\end{proof}

\subsection{Proof of Claim~\ref{clm:consistent}}
\label{app:proof-clm-consistent}
\begin{repclm}{clm:consistent}
For any scoring function $s$ that realizes the $k$-margin over a dataset we have,
\[
\ell^\avg_\preck(s) = \preck(s) = 0.
\]
\end{repclm}
\begin{proof}
We shall prove that for any $\hy$ such that $\norm{\hy}_1 = k$, under the $k$-margin condition, we have $\Delta(s,\hy) = 0$. This will show us that $\ell^\avg_\preck(s) = \max_{\norm{\hy}_1=k}\Delta(s,\hy) = 0$. Using Claim~\ref{clm:upper-bd} and the fact that $\preck(s) \geq 0$ will then prove the claimed result. We will analyze two cases in order to do this

\textbf{Case 1} ($K(\y,\hy) = k$): In this case the labeling $\hy$ is able to identify $k$ relevant points correctly and thus we have $C(\hy) = 1$ and we have
\[
\Delta(s,\hy) = \Delta(\y,\hy) + \sum_{i=1}^ns_i(\hy_i-\y_i) + \sum_{i=1}^n(1 - \hy_i)\y_is_i
\]
Now, since $K(\y,\hy) = k$, we have $\Delta(\y,\hy) = 0$ which means for all $i$ such that $\hy_i = 1$, we also have $\y_i = 1$. Thus, we have $\hy_i = \hy_i\y_i$. Thus,
\[
\Delta(s,\hy) = 0 + \sum_{i=1}^ns_i(\hy_i-\y_i) + \sum_{i=1}^n(\y_i - \hy_i\y_i)s_i = \sum_{i=1}^ns_i(\hy_i-\y_i) + \sum_{i=1}^n(\y_i - \hy_i)s_i = 0
\]

\textbf{Case 2} ($K(\y,\hy) = k' < k$): In this case, $\hy$ contains false positives. Thus we have
\begin{eqnarray*}
\Delta(s,\hy) &=& \Delta(\y,\hy) + \sum_{i=1}^ns_i(\hy_i-\y_i) + \frac{n_+-k}{n_+-k'}\sum_{i=1}^n(1 - \hy_i)\y_is_i\\
			  &=& \Delta(\y,\hy) + \sum_{i=1}^n\hy_i(1-\y_i)s_i - \frac{k-k'}{n_+-k'}\sum_{i=1}^n\y_i(1 - \hy_i)s_i\\
			  &=& (k-k')\br{\underbrace{\frac{1}{k-k'}\Delta(\y,\hy)}_{(A)} + \underbrace{\frac{1}{k-k'}\sum_{i=1}^n\hy_i(1-\y_i)s_i}_{(B)} - \underbrace{\frac{1}{n_+-k'}\sum_{i=1}^n\y_i(1 - \hy_i)s_i}_{(C)}}
\end{eqnarray*}
Now we have, by definition, $(A) = 1$. We also have
\[
(B) = \frac{1}{k-k'}\sum_{i=1}^n\hy_i(1-\y_i)s_i \leq \max_{j : \y_j = 0} s_j,
\]
as well as
\begin{eqnarray*}
(C) &=& \frac{1}{n_+-k'}\sum_{i=1}^n\y_i(1 - \hy_i)s_i\\
	&\geq& \min_{\substack{S_+ \subseteq \vecX_+\\|S_+| = n_+-k'}}\frac{1}{n_+-k'}\sum_{i \in S_+}\y_i(1 - \hy_i)s_i\\
	&\geq& \min_{\substack{S_+ \subseteq \vecX_+\\|S_+| = n_+-k+1}}\frac{1}{n_+-k+1}\sum_{i \in S_+}\y_i(1 - \hy_i)s_i,
\end{eqnarray*}
where the last step follows from Lemma~\ref{lem:rank-ineq} and the fact that $k' \leq k-1$ in this case analysis. Then we have
\[
\Delta(s,\hy) = (k-k')((A) + (B) - (C)) \leq (k-k')\br{1 + \max_{j : \y_j = 0} s_j - \min_{\substack{S_+ \subseteq \vecX_+\\|S_+| = n_+-k+1}}\frac{1}{n_+-k+1}\sum_{i \in S_+}\y_i(1 - \hy_i)s_i} \leq 0
\]
where the last step follows because $s$ realizes the $k$-margin. Having exhausted all cases, we establish the claim.
\end{proof}

\section{Proofs from Section~\ref{sec:perceptron}}\label{app:prf_perc}
\subsection{Proof of Theorem~\ref{thm:ptron-mistake}}
\begin{repthm}{thm:ptron-mistake}
Suppose $\norm{\x_t^i} \leq R$ for all $t,i$. Let $\Delta^C_T = \sum_{t=1}^T\Delta_t$ be the cumulative observed mistake values when Algorithm~\ref{algo:ptronATk} is run. Also, for any predictor $\w$, let $\hat\L_T(\w) = \sum_{t=1}^T \ell^\avg_\preck(\w;\vecX_t,\y_t)$. Then we have
\[
\Delta^C_T \leq \min_{\w}\ {\br{\norm{\w}\cdot R\cdot\sqrt{4k} + \sqrt{\hat\L_T(\w)}}^2}.
\]
\end{repthm}
\begin{proof}
We will prove the theorem using two lemmata that we state below.
\begin{lem}
\label{lem:ptron-mistake-proof-part1}
For any time step $t$, we have
\[
\norm{\w_t}^2 \leq \norm{\w_{t-1}}^2 + 4kR^2\Delta_t
\]
\end{lem}

\begin{lem}
\label{lem:ptron-mistake-proof-part2}
For any fixed $\w \in \W$, define $P_t := \ip{\w_t}{\w}$. Then we have
\[
P_t \geq P_{t-1} + \Delta_t - \ell^\avg_\preck(\w;\vecX_t,\y_t).
\]
\end{lem}

Using Lemmata~\ref{lem:ptron-mistake-proof-part1}~and~\ref{lem:ptron-mistake-proof-part2}, we can establish the mistake bound as follows. A repeated application of Lemma~\ref{lem:ptron-mistake-proof-part2} tells us that
\[
P_T \geq \sum_{t=1}^T\Delta_t - \sum_{t=1}^T\ell^\avg_\preck(\w;\vecX_t,\y_t) = \Delta^C_t - \hat\L_T(\w).
\]
In case the right hand side is negative, we already have the result with us. In case it is positive, we can now analyze further using the Cauchy-Schwartz inequality, and a repeated application of Lemma~\ref{lem:ptron-mistake-proof-part1}. Starting from the above we have
\begin{eqnarray*}
\Delta_T^C &\leq& P_T + \hat\L_T(\w)\\
								  &=& \ip{\w_T}{\w} + \hat\L_T(\w)\\
								  &\leq& \norm{\w_T}\norm{\w} + \hat\L_T(\w)\\
								  &\leq& \norm{\w}\sqrt{4kR^2\cdot\Delta_T^C} + \hat\L_T(\w),
\end{eqnarray*}
which gives us the desired result upon solving the quadratic inequality\footnote{More specifically, we use the fact that the inequality $(x - l)^2 \leq cx$ has a solution $x \leq (\sqrt l + \sqrt c)^2$ whenever $x,l,c \geq 0$ and $x \geq l$.}. We now prove the lemmata below. Note that in the following discussion, we have, for sake of brevity, used the notation $\hy = \hy_t = \y^{(\w_{t-1},k)}$.
\begin{proof}[Proof of Lemma~\ref{lem:ptron-mistake-proof-part1}]
For time steps where $\Delta_t = 0$, the result obviously holds since $\w_t = \w_{t-1}$. For analyzing other time steps, let $\v_t = D_t\cdot\sum_{i \in [b]} (1-\hy_i)\y_i\cdot\x_t^i - \sum_{i \in [b]} (1-\y_i)\hy_i\cdot\x_t^i$ so that $\w_t = \w_{t-1} + \v_t$. This gives us
\[
\norm{\w_t}^2 = \norm{\w_{t-1}}^2 + 2\ip{\w_{t-1}}{\v_t} + \norm{\v_t}^2.
\]
Let $s_i = \w_{t-1}^\top\x_t^i$. Then we have
\begin{eqnarray*}
\ip{\w_{t-1}}{\v_t} &=& D_t\cdot\sum_{i \in [b]} (1-\hy_i)\y_is_i - \sum_{i \in [b]} (1-\y_i)\hy_is_i\\
					&=& \Delta_t\br{\underbrace{\frac{1}{\norm{\y_t}_1 - K(\y_t,\hy_t)} \sum_{i \in [b]} (1-\hy_i)\y_is_i}_{(A)} - \underbrace{\frac{1}{\Delta_t}\sum_{i \in [b]} (1-\y_i)\hy_is_i}_{(B)}}\\
					&\leq& 0,
\end{eqnarray*}
where the last step follows since $(A)$ is the average of scores given to the false negatives and $(B)$ is the average of scores given to the false positives and by the definition of $\hy_t$, since false negatives are assigned scores less than false positives, we have $(A) \leq (B)$.
We also have
\begin{eqnarray*}
\norm{\v_t}^2 &=& \Delta_t^2\norm{\frac{1}{\norm{\y_t}_1 - K(\y_t,\hy_t)}\cdot\sum_{i \in [b]} (1-\hy_i)\y_i\cdot\x_t^i - \frac{1}{\Delta_t}\sum_{i \in [b]} (1-\y_i)\hy_i\cdot\x_t^i}^2\\
			  &\leq& 4\Delta_t^2R^2 \leq 4kR^2\Delta_t,
\end{eqnarray*}
since $\Delta_t \leq k$. Combining the two gives us the desired result.
\end{proof}
\begin{proof}[Proof of Lemma~\ref{lem:ptron-mistake-proof-part2}]
We prove the result using two cases. For sake of convenience, we will refer to $\y_t$ and $\hy_t$ as $\y$ and $\hy$ respectively.

\textbf{Case 1} ($\Delta_t = 0$): In this case $P_t = P_{t-1}$ since the model is not updated. However, since $\ell^\avg_\preck(\w) \geq \preck(\w) \geq 0$ for all $\w \in \W$ (by Claim~\ref{clm:upper-bd}), we still get
\[
P_t \geq P_{t-1} - \ell^\avg_\preck(\w;\vecX_t,\y_t),
\]
as required.

\textbf{Case 2} ($\Delta_t > 0$): In this case we use the update to $\w_{t-1}$ to evaluate the update to $P_{t-1}$. For sake of convenience, let us use the notation $s_i = \w^\top\x_t^i$. Also note that in Algorithm~\ref{algo:ptronATk}, $D_t = 1 - \frac{1}{C(\hy)}$.
\begin{eqnarray*}
P_t &=& P_{t-1} - \sum_{i \in [b]}(1 - \y_i)\hy_i s_i + D_t\cdot\sum_{i \in [b]}(1 - \hy_i)\y_i s_i\\
	&=& P_{t-1} - \sum_{i \in [b]}(1 - \y_i)\hy_i s_i + \br{1 - \frac{1}{C(\hy)}}\sum_{i \in [b]}(1 - \hy_i)\y_i s_i\\
	&=& P_{t-1} - \underbrace{\br{\sum_{i \in [b]}(\hy_i - \y_i)s_i + \frac{1}{C(\hy)}\sum_{i \in [b]}(1 - \hy_i)\y_i s_i}}_{(Q)}\\
	&\geq& P_{t-1} + \Delta_t - \ell^\avg_\preck(\w;\vecX_t,\y_t),
\end{eqnarray*}
where the last step follows from the definition of $\ell^\avg_\preck(\cdot)$ which gives us
\begin{align*}
\Delta_t + (Q) &= \Delta(\y,\hy) + \sum_{i \in [b]}(\hy_i - \y_i)s_i + \frac{1}{C(\hy)}\sum_{i \in [b]}(1 - \hy_i)\y_i s_i\\
			   &\leq \max_{\norm{\hy}_1=k}\ \bc{\Delta(\y,\hy) + \sum_{i \in [b]}s_i(\hy_i-\y_i) + \frac{1}{C(\hy)}\sum_{i \in [b]}(1 - \hy_i)\y_is_i}\\
			   &= \ell^\avg_\preck(s) = \ell^\avg_\preck(\w;\vecX_t,\y_t)\qedhere
\end{align*}
\end{proof}

This concludes the proof of the mistake bound.
\end{proof}

\subsection{Proof of Theorem~\ref{thm:ptron-fast-mistake}}
\begin{repthm}{thm:ptron-fast-mistake}
Suppose $\norm{\x_t^i} \leq R$ for all $t,i$. Let $\Delta^C_T = \sum_{t=1}^T\Delta_t$ be the cumulative observed mistake values when Algorithm~\ref{algo:ptronATk-fast} is run. Also, for any predictor $\w$, let $\hat\L^{\mx}_T(\w) = \sum_{t=1}^T \ell^{\mx}_\preck(\w;\vecX_t,\y_t)$. Then we have
\[
\Delta^C_T \leq \min_{\w}\ {\br{\norm{\w}\cdot R\cdot\sqrt{4k} + \sqrt{\hat\L^{\mx}_T(\w)}}^2}.
\]
\end{repthm}
\begin{proof}
As before, we will prove this theorem in two parts. Lemma~\ref{lem:ptron-mistake-proof-part1} will continue to hold in this case as well. However, we will need a modified form of Lemma~\ref{lem:ptron-mistake-proof-part2} that we prove below. As before, we will use the notation $\hy = \hy_t = \y^{(\w_{t-1},k)}$.

\begin{lem}
\label{lem:ptron-fast-mistake-proof-part2}
For any fixed $\w \in \W$, define $P_t := \ip{\w_t}{\w}$. Then we have
\[
P_t \geq P_{t-1} + \Delta_t - \ell^{\mx}_\preck(\w;\vecX_t,\y_t).
\]
\end{lem}

Using Lemmata~\ref{lem:ptron-mistake-proof-part1}~and~\ref{lem:ptron-fast-mistake-proof-part2}, the theorem follows as before. All that remains now is to prove Lemma~\ref{lem:ptron-fast-mistake-proof-part2}.
\begin{proof}[Proof of Lemma~\ref{lem:ptron-fast-mistake-proof-part2}]
We prove the result using two cases as before. For sake of convenience, we will refer to $\y_t$ and $\hy_t$ as $\y$ and $\hy$ respectively.

\textbf{Case 1} ($\Delta_t = 0$): In this case $P_t = P_{t-1}$ since the model is not updated. However, since $\ell^{\mx}_\preck(\w) \geq \preck(\w) \geq 0$ for all $\w \in \W$ (by Claim~\ref{clm:upper-bd-tight}), we still get
\[
P_t \geq P_{t-1} - \ell^{\mx}_\preck(\w;\vecX_t,\y_t),
\]
as required.

\textbf{Case 2} ($\Delta_t > 0$): In this case we use the update to $\w_{t-1}$ to evaluate the update to $P_{t-1}$. For sake of convenience, let us use the notation $s_i = \w^\top\x_t^i$. Also note that the set $S_t := \FN(\w^{t-1},\Delta_t)$ contains the false negatives in the top $\Delta_t$ positions as ranked by $\w^{t-1}$.
\begin{eqnarray*}
P_t &=& P_{t-1} - \sum_{i \in [b]}(1 - \y_i)\hy_i s_i + \sum_{i \in S_t}(1 - \hy_i)\y_i s_i\\
	&=& P_{t-1} - \sum_{i \in [b]}(1 - \y_i)\hy_i s_i - \sum_{i \in [b]}\y_i\hy_i s_i + \sum_{i \in [b]}\y_i\hy_i s_i + \sum_{i \in S_t}(1 - \hy_i)\y_i s_i\\
	&=& P_{t-1} - \sum_{i \in [b]}\hy_i s_i + \sum_{i \in [b]}\y_i\hy_i s_i + \sum_{i \in S_t}(1 - \hy_i)\y_i s_i\\
	&=& P_{t-1} - \br{\sum_{i \in [b]}(\hy_i - \y_i) s_i + \sum_{i \in [b]}(1 - \hy_i)\y_is_i - \sum_{i \in S_t}(1 - \hy_i)\y_i s_i}\\
	&\geq& P_{t-1} - \underbrace{\br{\sum_{i \in [b]}(\hy_i - \y_i) s_i + \max_{\substack{\tilde\y \preceq (1-\hy)\cdot\y \\ \norm{\tilde\y}_1 = n_+ - k}}\sum_{i=1}^n\tilde\y_is_i}}_{(Q)}\\
	&\geq& P_{t-1} + \Delta_t - \ell^{\mx}_\preck(\w;\vecX_t,\y_t),
\end{eqnarray*}
where the last step follows from the definition of $\ell^\avg_\preck(\cdot)$ which gives us
\begin{align*}
\Delta_t + (Q) &= \Delta_t + \sum_{i \in [b]}(\hy_i - \y_i) s_i + \max_{\substack{\tilde\y \preceq (1-\hy)\cdot\y \\ \norm{\tilde\y}_1 = n_+ - k}}\sum_{i=1}^n\tilde\y_is_i\\
			   &\leq \max_{\norm{\hy}_1=k}\ \bc{\Delta_t + \sum_{i \in [b]}(\hy_i - \y_i) s_i + \max_{\substack{\tilde\y \preceq (1-\hy)\cdot\y \\ \norm{\tilde\y}_1 = n_+ - k}}\sum_{i=1}^n\tilde\y_is_i}\\
			   &= \ell^\mx_\preck(s) = \ell^{\mx}_\preck(\w;\vecX_t,\y_t)\qedhere
\end{align*}
\end{proof}
This concludes the proof of the theorem.
\end{proof}

\section{Proof of Theorem~\ref{thm:uc-preck-surrogates}}
\label{app:uc-bounds-surrogates}
Our proof of Theorem~\ref{thm:uc-preck-surrogates} crucially utilizes the following two lemmas that helps in exploiting the structure in our surrogate functions. 
The first basic lemma states that the pointwise supremum of a set of Lipschitz functions is also Lipschitz.
\begin{lem}
\label{lem:sup-lip}
Let $f_1,\ldots,f_m$ be $m$ real valued functions $f_i: \R^n \> \R$ such that every $f_i$ is $1$-Lipschitz with respect to the $\norm{\cdot}_\infty$ norm. Then the function
\[
g(\v) = \max_{i\in[m]}\ f_i(\v)
\]
is $1$-Lipschitz with respect to the $\norm{\cdot}_\infty$ norm too.
\end{lem}
The second lemma establishes the convergence of additive estimates over the top of ranked lists. The abstract nature of the result would allow us to apply it to a wide variety of situations and would be crucial to our analyses.
\begin{lem}
\label{lem:rank-conv}
Let $\V$ be a universe with a total order $\succeq$ established on it and let $\v_1,\ldots,\v_n$ be a population of $n$ items arranged in decreasing order. Let $\hat \v_1,\ldots,\hat \v_b$ be a sample chosen i.i.d. (or without replacement) from the population and arranged in decreasing order as well. Then for any fixed $h: \V \> [-1,1]$ and $\kappa \in (0,1]$, we have, with probability at least $1 - \delta$ over the choice of the samples,
\[
\abs{\frac{1}{\ceil{\kappa n}}\sum_{i=1}^{\ceil{\kappa n}}h(\v_i) - \frac{1}{\ceil{\kappa b}}\sum_{i=1}^{\ceil{\kappa b}}h(\hat \v_i)} \leq 4\sqrt\frac{\log\frac{2}{\delta}}{\kappa b}
\]
\end{lem}

\begin{repthm}{thm:uc-preck-surrogates}
The performance measure $\preckappa(\cdot)$, as well as the surrogates $\ell^\ramp_\preckappa(\cdot)$, $\ell^\avg_\preckappa(\cdot)$ and $\ell^\mx_\preckappa(\cdot)$, all exhibit uniform convergence at the rate $\alpha(b,\delta) = \O{\sqrt{\frac{1}{b}\log\frac{1}{\delta}}}$.
\end{repthm}

We will prove the four parts of the theorem in three separate subsections below. We shall consider a population $\z_1,\ldots,\z_n$ and a sample of size $b$ $\hz_1,\ldots,\hz_b$ chosen uniformly at random with (i.e. i.i.d.) or without replacement. We shall let $p$ and $\hat p$ denote the fraction of positives in the population and the sample respectively. In the following, we shall reserve the notation $\hy$ for the label vector in the sample and shall use the notation $\ty$ to denote candidate labellings in the definition of the surrogate.

\subsection{A Uniform Convergence Bound for the $\preckappa(\cdot)$ Performance Measure}
We note that a point-wise convergence result for $\preckappa(\cdot)$ follows simply from Lemma~\ref{lem:rank-conv}. To see this, given a population $\z_1,\ldots,\z)n$ and a fixed model $\w \in \W$, construct a parallel population using the transformation $\v_i \< (\w^\top\x_i,\y_i) \in \R^2$. We order these tuples according to their first component, i.e. along the scores and use $h(\v_i) = 1 - \y_i$. Let the population be arranged such that $\v_1 \succeq \v_2 \succeq \ldots$. Then this gives us
\[
\sum_{i=1}^kh(\v_i) = \sum_{i=1}^k (1 - \y_i) = \preck(\y,\y^{(\w,k)}) = \preck(\w).
\]
Thus, the application of Lemma~\ref{lem:rank-conv} gives us the following result
\begin{lem}
For any fixed model $\w\in\W$, with probability at least $1 - \delta$ over the choice of $b$ samples, we have
\[
\abs{\preckappa(\w;\z_1,\ldots,\z_n) - \preckappa(\w;\hz_1,\ldots,\hz_b)} \leq \O{\sqrt{\frac{1}{b}\log\frac{1}{\delta}}}.
\]
\end{lem}
To prove the uniform convergence result, we will, in some sense, require a uniform version of Lemma~\ref{lem:rank-conv}. To do so we fix some notation. For any fixed $\kappa > 0$, and for any $\w \in \W$, we will define $v_\w$ as the largest real number $v$ such that
\[
\sum_{i=1}^n\Ind{\w^\top\x_i \geq v} = \kappa p n
\]
Similarly, we will define $\hat v_\w$ as the largest real number $v$ such that
\[
\sum_{i=1}^b\Ind{\w^\top\hx_i \geq v} = \kappa\hat p b
\]
Using this notation we can redefine $\preckappa(\cdot)$ on the population, as well as the sample, as
\begin{align*}
\preckappa(\w;\z_1,\ldots,\z_n) &:= \frac{1}{\kappa p n}\sum_{i=1}^{n}\Ind{\w^\top\x \geq v_\w}\cdot\Ind{\y_i = 0}\\
\preckappa(\w;\hz_1,\ldots,\hz_b) &:= \frac{1}{\kappa\hat p b}\sum_{i=1}^{b}\Ind{\w^\top\x \geq \hat v_\w}\cdot\Ind{\hy_i = 0}
\end{align*}
We can now write
\begin{align*}
{}&\sup_{\w\in\W}\abs{\preckappa(\w;\z_1,\ldots,\z_n) - \preckappa(\w;\hz_1,\ldots,\hz_b)}\\
={}& \sup_{\w\in\W}\abs{\frac{1}{\kappa p n}\sum_{i=1}^{n}\Ind{\w^\top\x \geq v_\w}\cdot\Ind{\y_i = 0} - \frac{1}{\kappa\hat p b}\sum_{i=1}^{b}\Ind{\w^\top\x \geq \hat v_\w}\cdot\Ind{\hy_i = 0}}\\
\leq{}& \sup_{\w\in\W}\abs{\frac{1}{\kappa p n}\sum_{i=1}^{n}\Ind{\w^\top\x \geq v_\w}\cdot\Ind{\y_i = 0} - \frac{1}{\kappa \hat p b}\sum_{i=1}^{b}\Ind{\w^\top\x \geq v_\w}\cdot\Ind{\hy_i = 0}}\\
& + \sup_{\w\in\W}\abs{\frac{1}{\kappa \hat p b}\sum_{i=1}^{b}\Ind{\w^\top\x \geq v_\w}\cdot\Ind{\hy_i = 0} - \frac{1}{\kappa\hat p b}\sum_{i=1}^{b}\Ind{\w^\top\x \geq \hat v_\w}\cdot\Ind{\hy_i = 0}}\\
\leq{}& \underbrace{\sup_{\w\in\W,t\in\R}\abs{\frac{1}{\kappa p n}\sum_{i=1}^{n}\Ind{\w^\top\x \geq t}\cdot\Ind{\y_i = 0} - \frac{1}{\kappa \hat p b}\sum_{i=1}^{b}\Ind{\w^\top\x \geq t}\cdot\Ind{\hy_i = 0}}}_{(A)}\\
& + \underbrace{\sup_{\w\in\W}\abs{\frac{1}{\kappa \hat p b}\sum_{i=1}^{b}\Ind{\w^\top\x \geq v_\w}\cdot\Ind{\hy_i = 0} - \frac{1}{\kappa\hat p b}\sum_{i=1}^{b}\Ind{\w^\top\x \geq \hat v_\w}\cdot\Ind{\hy_i = 0}}}_{(B)}\\
\end{align*}
Now, using a standard VC-dimension based uniform convergence argument over the class of thresholded classifiers, we get the following result: with probability at least $1 - \delta$
\[
(A) \leq \O{\sqrt{\frac{1}{b}\br{\log\frac{1}{\delta}+d_{\text{VC}}(\W)\cdot\log b}}} = \softO{\sqrt{\frac{1}{b}\log\frac{1}{\delta}}},
\]
where $d_{\text{VC}}(\W)$ is the VC-dimension of the set of classifiers $\W$. Moving on to bound the second term, we can use an argument similar to the one used to prove Lemma~\ref{lem:rank-conv} to show that
\begin{align*}
(B) &\leq \sup_{\w\in\W}\abs{\frac{1}{\kappa \hat p b}\sum_{i=1}^{b}\Ind{\w^\top\x \geq v_\w} - \frac{1}{\kappa\hat p b}\sum_{i=1}^{b}\Ind{\w^\top\x \geq \hat v_\w}}\\
	&\leq \sup_{\w\in\W}\abs{\frac{1}{\kappa \hat p b}\sum_{i=1}^{b}\Ind{\w^\top\x \geq v_\w} - \kappa}\\
	&\leq \sup_{\w\in\W}\abs{\frac{1}{\kappa \hat p b}\sum_{i=1}^{b}\Ind{\w^\top\x \geq v_\w} - \frac{1}{\kappa p n}\sum_{i=1}^{n}\Ind{\w^\top\x \geq v_\w}}\\
	&\leq \softO{\sqrt{\frac{1}{b}\log\frac{1}{\delta}}},
\end{align*}
where the last step follows from a standard VC-dimension based uniform convergence argument as before. This establishes the following uniform convergence result for the $\preck(\cdot)$ performance measure
\begin{thm}
We have, with probability at least $1 - \delta$ over the choice of $b$ samples,
\[
\sup_{\w\in\W}\abs{\preckappa(\w;\z_1,\ldots,\z_n) - \preckappa(\w;\hz_1,\ldots,\hz_b)} \leq \softO{\sqrt{\frac{1}{b}\log\frac{1}{\delta}}}.
\]
\end{thm}

\subsection{A Uniform Convergence Bound for the $\ell^\ramp_\preckappa(\cdot)$ Surrogate}
We first recall the form of the (normalized) surrogate below - note that this is a non-convex surrogate. Also recall that $k = \kappa\cdot n_+(\y)$.
\[
\ell^\ramp_\preckappa(\w;\ \z_1,\ldots,\z_n) = \underbrace{\max_{\norm{\ty}_1=k}\bc{\frac{\Delta(\y,\ty)}{k} + \frac{1}{k}\sum_{i=1}^n\ty_i\w^\top\x_i}}_{\Psi_1(\w;\ \z_1,\ldots,\z_n)} - \underbrace{\max_{\substack{\norm{\ty}_1 = k\\K(\y,\tilde\y) = k}}\frac{1}{k}\sum_{i=1}^n\tilde\y_i\w^\top\x_i}_{\Psi_2(\w;\ \z_1,\ldots,\z_n)}
\]
We will now show that both the functions $\Psi_1(\cdot)$, as well as $\Psi_2(\cdot)$, exhibit uniform convergence. This shall suffice to prove that $\ell^\ramp_\preckappa(\cdot)$ exhibits uniform convergence. To do so we shall show that the two functions exhibit pointwise convergence and that they are Lipschitz. This will allow a standard $L_\infty$ covering number argument \cite{covering-numbers-zhang} to give us the required uniform convergence results.

\subsubsection{A Uniform Convergence Result for $\Psi_1(\cdot)$}
We have
\begin{align*}
\Psi_1(\w;\ \z_1,\ldots,\z_n) &= \max_{\norm{\ty}_1=\kappa p n}\bc{\frac{1}{\kappa p n}\sum_{i=1}^n\ty_i(\w^\top\x_i - \y_i)} + 1\\
\Psi_1(\w;\ \hz_1,\ldots,\hz_b) &= \max_{\norm{\ty}_1=\kappa \hat p b}\bc{\frac{1}{\kappa \hat p b}\sum_{i=1}^b\ty_i(\w^\top\hx_i - \hy_i)} + 1
\end{align*}

An application of Corollary~\ref{cor:sup-lip-model} indicates that $\Psi_1(\cdot)$ is Lipschitz i.e.
\[
\abs{\Psi_1(\w;\ \z_1,\ldots,\z_n) - \Psi_1(\w';\ \z_1,\ldots,\z_n)} \leq \O{\norm{\w-\w'}_2}.
\]
Thus, all that remains is to prove pointwise convergence. We decompose the error as follows
\begin{align*}
\abs{\Psi_1(\w;\ \z_1,\ldots,\z_n) - \Psi_1(\w;\ \hz_1,\ldots,\hz_b)} \leq {} &\underbrace{\abs{\Psi_1(\w;\ \z_1,\ldots,\z_n) - \max_{\norm{\ty}_1=\kappa p b}\bc{\frac{1}{\kappa p b}\sum_{i=1}^b\ty_i(\w^\top\hx_i - \hy_i)} + 1}}_{(A)}\\
					& + \underbrace{\abs{\max_{\norm{\ty}_1=\kappa p b}\bc{\frac{1}{\kappa p b}\sum_{i=1}^b\ty_i(\w^\top\hx_i - \hy_i)} + 1 - \Psi_1(\w;\ \hz_1,\ldots,\hz_b)}}_{(B)}
\end{align*}
An application of Lemma~\ref{lem:rank-conv} using $\v_i = \w^\top\hx_i - \hy_i$ and $h(\cdot)$ as the identity function shows us that
\[
(A) \leq \O{\frac{1}{\kappa p}\sqrt{\frac{1}{b}\log\frac{1}{\delta}}}.
\]
To bound the residual term $(B)$, notice that an application of the Hoeffding's inequality tells us that with probability at least $1 - \delta$
\[
\abs{p - \hat p} \leq \sqrt{\frac{1}{2b}\log\frac{2}{\delta}},
\]
which lets us bound the residual as follows. Assume, for sake of simplicity, that the sample data points have been ordered in decreasing order of the quantity $\w^\top\hx_i - \y_i$ as well as that $\abs{\w^\top\x} \leq 1$ for all $\x$.
\begin{align*}
(B) &= \abs{\max_{\norm{\ty}_1=\kappa p b}\bc{\frac{1}{\kappa p b}\sum_{i=1}^b\ty_i(\w^\top\hx_i - \hy_i)} - \max_{\norm{\ty}_1=\kappa \hat p b}\bc{\frac{1}{\kappa \hat p b}\sum_{i=1}^b\ty_i(\w^\top\hx_i - \hy_i)}}\\
	&= \abs{{\frac{1}{\kappa p b}\sum_{i=1}^{\kappa p b}(\w^\top\hx_i - \hy_i)} - {\frac{1}{\kappa \hat p b}\sum_{i=1}^{\kappa \hat p b}(\w^\top\hx_i - \hy_i)}}\\
	&\leq \abs{\sum_{i=1}^{\kappa \min\bc{p,\hat p} b}\br{\frac{1}{\kappa p b} - \frac{1}{\kappa \hat p b}}(\w^\top\hx_i - \hy_i)} + \abs{\frac{1}{\kappa \max\bc{p,\hat p} b}\sum_{i=\kappa\min\bc{p,\hat p}b+1}^{\kappa \max\bc{p,\hat p} b}(\w^\top\hx_i - \hy_i)}\\
	&\leq \frac{2}{\kappa b}\abs{\frac{p - \hat p}{p\hat p}}\cdot\kappa\min\bc{p,\hat p}b+\frac{2}{\kappa \max\bc{p,\hat p} b}\cdot\kappa\abs{p-\hat p}b\\
	&= 2\abs{p -\hat p}\cdot\br{\frac{\min\bc{p,\hat p}}{p\hat p} + \frac{1}{\max\bc{p, \hat p}}}\\
	&\leq \sqrt{\frac{1}{2b}\log\frac{2}{\delta}}\cdot\frac{2}{\max\bc{p,\hat p}} \leq \frac{2}{p}\sqrt{\frac{1}{2b}\log\frac{2}{\delta}}
\end{align*}
This establishes that for any fixed $\w \in \W$, with probability at least $1 - \delta$, we have
\[
\abs{\Psi_1(\w;\ \z_1,\ldots,\z_n) - \Psi_1(\w;\ \hz_1,\ldots,\hz_b)} \leq \O{\sqrt{\frac{1}{b}\log\frac{1}{\delta}}}
\]
which concludes the uniform convergence proof.

\subsubsection{A Uniform Convergence Result for $\Psi_2(\cdot)$}
The proof follows similarly here with a direct application of Corollary~\ref{cor:sup-lip-model} showing us that $\Psi_2(\cdot)$ is Lipschitz and an application of Lemma~\ref{lem:rank-conv} along with the observation that $\abs{p - \hat p} \leq \sqrt{\frac{1}{2b}\log\frac{2}{\delta}}$ similar to the discussion used above concluding the point-wise convergence proof.

The above two part argument establishes the following uniform convergence result for the $\ell^\ramp_\preckappa(\cdot)$ performance measure
\begin{thm}
\label{lem:uc-proof-ramp-surrogate}
We have, with probability at least $1 - \delta$ over the choice of $b$ samples,
\[
\sup_{\w\in\W}\abs{\ell^\ramp_\preckappa(\w;\z_1,\ldots,\z_n) - \ell^\ramp_\preckappa(\w;\hz_1,\ldots,\hz_b)} \leq \O{\sqrt{\frac{1}{b}\log\frac{1}{\delta}}}.
\]
\end{thm}

\subsection{A Uniform Convergence Bound for the $\ell^\avg_\preckappa(\cdot)$ Surrogate}
This will be the most involved of the four bounds, given the intricate nature of the surrogate. We will prove this result using a series of partial results which we state below. As before, for any $\w \in \W$ and any $\ty$, we define
\begin{align*}
\Delta(\w,\ty) &:= \frac{1}{\kappa p n}\br{\Delta(\y,\ty) + \sum_{i=1}^n(\ty_i-\y_i)\w^\top\x_i + \frac{1}{C(\ty)}\sum_{i=1}^n(1 - \ty_i)\y_i\w^\top\x_i}\\
\hat\Delta(\w,\ty) &:= \frac{1}{\kappa \hat p b}\br{\Delta(\hy,\ty) + \sum_{i=1}^n(\ty_i-\hy_i)\w^\top\hx_i + \frac{1}{C(\ty)}\sum_{i=1}^n(1 - \ty_i)\hy_i\w^\top\hx_i}
\end{align*}
Recall that we are using $\hy$ to denote the true labels of the sample points and $\ty$ to denote the candidate labellings while defining the surrogates. We also define, for any $\beta \in [0,1]$, the following quantities
\begin{align*}
\Delta(\w,\beta) &:= \max_{\substack{\norm{\ty}_1 = \kappa p n\\K(\y,\ty) = \beta p n}}\bc{\Delta(\w,\ty)}\\
\hat\Delta(\w,\beta) &:= \max_{\substack{\norm{\ty}_1 = \kappa \hat p b\\K(\hy,\ty) = \beta \hat p b}}\bc{\hat\Delta(\w,\ty)}
\end{align*}
Note that $\beta$ denotes a target true positive \emph{rate} and consequently, can only take values between $0$ and $\kappa$.
Given the above, we claim the following lemmata
\begin{lem}
\label{lem:diff-beta-lip}
For every $\w$ and any $\beta, \beta' \in [0,\kappa]$, we have
\[
\abs{\Delta(\w,\beta) - \Delta(\w,\beta')} \leq \O{\abs{\beta - \beta'}}.
\]
\end{lem}
\begin{lem}
\label{lem:same-beta-uc}
For any fixed $\beta$, we have, with probability at least $1 - \delta$ over the choice of the sample
\[
\sup_{\w\in\W}\abs{\Delta(\w,\beta) - \hat\Delta(\w,\beta)} \leq \O{\sqrt{\frac{1}{b}\log\frac{1}{\delta}}}.
\]
\end{lem}
Using the above two lemmata as given, we can now prove the desired uniform convergence result for the $\ell^\avg_\preckappa(\cdot)$ surrogate:
\begin{thm}
With probability at least $1 - \delta$ over the choice of the samples, we have
\[
\sup_{\w\in\W}\abs{\ell^\avg_\preckappa(\w;\z_1,\ldots,\z_n) - \ell^\avg_\preckappa(\w;\hz_1,\ldots,\hz_b)} \leq \softO{\sqrt{\frac{1}{b}\log\frac{1}{\delta}}}.
\]
\end{thm}
\begin{proof}
We note that given the definitions of $\Delta(\w,\beta)$ and $\hat\Delta(\w,\beta)$, we can redefine the performance measure as follows
\[
\ell^\avg_\preckappa(\w;\z_1,\ldots,\z_n) = \max_{\beta \in [0,\kappa]}\Delta(\w,\beta)
\]
We now note that for the population, the set of achievable values of true positive rates i.e. $\beta$ is
\[
B = \bc{0,\frac{1}{\kappa p n},\frac{2}{\kappa p n},\ldots,\frac{\kappa pn - 1}{\kappa pn}, 1},
\]
which correspond, respectively, to classifiers for which the \emph{number} of true positives equals $\bc{0,1,2\ldots\kappa p n-1, \kappa p n}$. Similarly, the set of achievable values of true positive rates i.e. $\beta$ for the sample is
\[
\hat B = \bc{0,\frac{1}{\kappa \hat p b},\frac{2}{\kappa \hat p b},\ldots,\frac{\kappa \hat p b - 1}{\kappa \hat p b}, 1}.
\]
Clearly, for any $\beta \in B$, there exists a $\pi_{\hat B}(\beta) \in \hat B$ such that
\[
\abs{\pi_{\hat B}(\beta) - \beta} \leq \frac{1}{\kappa \hat p b}.
\]
Given this, let us define
\begin{align*}
\beta^\ast(\w) &= \arg\max_{\beta \in [0,\kappa]}\Delta(\w,\beta)\\
\hat\beta^\ast(\w) &= \arg\max_{\hat\beta \in [0,\kappa]}\hat\Delta(\w,\hat\beta)
\end{align*}
We shall assume, for the sake of simplicity, that $s|n$ so that $\hat B \subset B$. This gives us the following set of inequalities for any $\w\in\W$:
\begin{align*}
\Delta(\w,\beta^\ast(\w)) &\leq \Delta(\w,\pi_{\hat B}(\beta^\ast(\w))) + \abs{\beta^\ast(\w) - \pi_{\hat B}(\beta^\ast(\w))}\\
						  &\leq \hat\Delta(\w,\pi_{\hat B}(\beta^\ast(\w))) + \sup_{\w\in\W}\abs{\Delta(\w,\pi_{\hat B}(\beta^\ast(\w))) - \hat\Delta(\w,\pi_{\hat B}(\beta^\ast(\w)))} + \frac{1}{\kappa\hat p b}\\
						  &\leq \hat\Delta(\w,\pi_{\hat B}(\beta^\ast(\w))) + \sup_{\w\in\W, \hat\beta \in \hat B}\abs{\Delta(\w,\hat\beta) - \hat\Delta(\w,\hat\beta)} + \frac{1}{\kappa\hat p b}\\
						  &\leq \hat\Delta(\w,\pi_{\hat B}(\beta^\ast(\w))) + \O{\sqrt{\frac{1}{b}\log\frac{b}{\delta}}} + \frac{1}{\kappa\hat p b}\\
						  &\leq \hat\Delta(\w,\hat\beta^\ast(\w)) + \O{\sqrt{\frac{1}{b}\log\frac{b}{\delta}}} + \frac{1}{\kappa\hat p b},
\end{align*}
where the first step follows from Lemma~\ref{lem:diff-beta-lip}, the third step follows since $\pi_{\hat B}(\beta^\ast(\w)) \in \hat B$, the fourth step follows from an application of the union bound with Lemma~\ref{lem:same-beta-uc} over the set of elements in $\hat B$ and noting $\abs{\hat B} \leq \O{b}$, and the last step follows from the optimality of $\hat\beta^\ast(\w)$. Similarly we can write, for any $\w\in\W$,
\begin{align*}
\hat\Delta(\w,\hat\beta^\ast(\w)) &\leq \Delta(\w,\hat\beta^\ast(\w)) + \O{\sqrt{\frac{1}{b}\log\frac{b}{\delta}}}\\
								  &\leq \Delta(\w,\beta^\ast(\w)) + \O{\sqrt{\frac{1}{b}\log\frac{b}{\delta}}},
\end{align*}
where the first step uses Lemma~\ref{lem:same-beta-uc} with a union bound over elements in $\hat B$ and the fact that $\hat\beta^\ast(\w) \in \hat B \subset B$ (note that this assumption is not crucial to the argument -- indeed, even if $\hat\beta^\ast(\w) \notin B$, we would only incur an extra $\O{\frac{1}{n}}$ error by an application of Lemma~\ref{lem:diff-beta-lip} since given the granularity of $B$, we would always be able to find a value in $B$ that is no more than $\O{\frac{1}{n}}$ far from $\hat\beta^\ast(\w)$), and the last step uses the optimality of $\beta^\ast(\w)$. Thus, we can write
\begin{align*}
\sup_{\w\in\W}\abs{\ell^\avg_\preckappa(\w;\z_1,\ldots,\z_n) - \ell^\avg_\preckappa(\w;\hz_1,\ldots,\hz_b)} &= \sup_{\w\in\W}\abs{\Delta(\w,\beta^\ast(\w)) - \hat\Delta(\w,\hat\beta^\ast(\w))}\\
&\leq \O{\sqrt{\frac{1}{b}\log\frac{b}{\delta}}} + \frac{1}{\kappa\hat p b}\\
&\leq \softO{\sqrt{\frac{1}{b}\log\frac{1}{\delta}}},
\end{align*}
since $\hat p \geq \Om{1}$ with probability at least $1 - \delta$. Thus, all we are left is to prove Lemmata~\ref{lem:diff-beta-lip}~and~\ref{lem:same-beta-uc} which we do below. To proceed with the proofs, we first write the form of $\Delta(\w,\beta)$ for a fixed $\w$ and $\beta$ and simplify the expression for ease of further analysis. We shall assume, for sake of simplicity, that $\beta pn, \kappa pn, \beta\hat pb$, and $\kappa\hat pb$ are all integers.
\begin{align*}
\Delta(\w,\beta) &= \max_{\substack{\norm{\ty}_1 = \kappa p n\\K(\y,\ty) = \beta p n}}\bc{\frac{1}{\kappa p n}\br{\Delta(\y,\ty) + \sum_{i=1}^n(\ty_i-\y_i)\w^\top\x_i + \frac{1}{C(\ty)}\sum_{i=1}^n(1 - \ty_i)\y_i\w^\top\x_i}}\\
				 &= 1 - \frac{\beta}{\kappa} - \underbrace{\frac{1}{\kappa p n}\br{\frac{\kappa-\beta}{1-\beta}}\sum_{i=1}^n\y_i\w^\top\x_i}_{A(\w,\beta)} + \underbrace{\max_{\substack{\norm{\ty}_1 = \kappa p n\\K(\y,\ty) = \beta p n}}\bc{\frac{1}{\kappa p n}{\sum_{i=1}^n\ty_i\br{1 - \frac{1-\kappa}{1-\beta}\cdot\y_i}\w^\top\x_i}}}_{B(\w,\beta)}
\end{align*}
We can similarly define $\hat A(\w,\beta)$ and $\hat B(\w,\beta)$ for the samples.
\begin{proof}[Proof of Lemma~\ref{lem:diff-beta-lip}]
We have, by the above simplification,
\[
\abs{\Delta(\w,\beta) - \Delta(\w,\beta')} = \frac{1}{\kappa}\abs{\beta - \beta'} + \abs{A(\w,\beta) - A(\w,\beta')} + \abs{B(\w,\beta) - B(\w,\beta')},
\]
as well as, assuming without loss of generality, that $\abs{\w^\top\x}\leq 1$ for all $\w$ and $\x$,
\begin{align*}
\abs{A(\w,\beta) - A(\w,\beta')} &\leq \abs{\frac{\kappa-\beta}{1-\beta} - \frac{\kappa-\beta'}{1-\beta'}}\cdot \abs{\frac{1}{\kappa pn}\sum_{i=1}^n\y_i\w^\top\x_i}\\
								 &\leq \frac{(1-\kappa)\abs{\beta-\beta'}}{\kappa(1-\beta)(1-\beta')} \leq \frac{1}{\kappa(1-\kappa)}\abs{\beta-\beta'},
\end{align*}
where the last step follows since $\beta,\beta' \leq \kappa$. To analyze the third term i.e. $\abs{B(\w,\beta) - B(\w,\beta')}$, we analyze the nature of the assignment $\ty$ which defines $B(\w,\beta)$. Clearly $\ty$ must assign $\beta pn$ positives and $(\kappa - \beta) pn$ negatives a label of $1$ and the rest, a label of $0$. Since it is supposed to maximize the scores thus obtained, it clearly assigns the top ranked $(\kappa - \beta)pn$ negatives a label of $1$. As far as positives are concerned, $\beta < \kappa$, we have $\br{1 - \frac{1-\kappa}{1-\beta}} \geq 0$ which means that the $\beta pn$ top ranked positives will get assigned a label of $1$.

To formalize this, let us set some notation. Let $s^+_1 \geq s^+_2 \geq \ldots \geq s^+_{pn}$ denote the scores of the positive points arranged in descending order. Similarly, let $s^-_1 \geq s^-_2 \geq \ldots \geq s^-_{(1-p)n}$ denote the scores of the negative points arranged in descending order. Given this notation, we can rewrite $B(\w,\beta)$ as follows:
\[
B(\w,\beta) = \frac{1}{\kappa pn}\br{\br{\frac{\kappa - \beta}{1 - \beta}}\sum_{i=1}^{\beta pn}s^+_i + \sum_{i=1}^{(\kappa -\beta) pn}s^-_i}.
\]
Thus, assuming without loss of generality that $\abs{s^+_i}, \abs{s^-_i} \leq 1$, we have,
\begin{align*}
\abs{B(\w,\beta) - B(\w,\beta')} &= \frac{1}{\kappa pn}\abs{\br{\frac{\kappa - \beta}{1 - \beta}}\sum_{i=1}^{\beta pn}s^+_i + \sum_{i=1}^{(\kappa -\beta) pn}s^-_i - \br{\frac{\kappa - \beta'}{1 - \beta'}}\sum_{i=1}^{\beta' pn}s^+_i - \sum_{i=1}^{(\kappa -\beta') pn}s^-_i}\\
								 &\leq \frac{1}{\kappa pn}\abs{\br{\frac{\kappa - \beta}{1 - \beta}}\sum_{i=1}^{\beta pn}s^+_i - \br{\frac{\kappa - \beta'}{1 - \beta'}}\sum_{i=1}^{\beta' pn}s^+_i} + \frac{1}{\kappa pn}\abs{\sum_{i=1}^{(\kappa -\beta) pn}s^-_i - \sum_{i=1}^{(\kappa -\beta') pn}s^-_i}\\
								 &\leq \abs{\frac{\kappa-\beta}{1-\beta} - \frac{\kappa-\beta'}{1-\beta'}}\cdot\abs{\frac{1}{\kappa pn}\sum_{i=1}^{\min\bc{\beta,\beta'}pn}s^+_i} + \frac{1}{\kappa pn}\frac{\kappa-\max\bc{\beta,\beta'}}{1-\max\bc{\beta,\beta'}}\abs{\beta - \beta'}pn + \frac{\abs{\beta-\beta'}pn}{\kappa pn}\\
								 &\leq \frac{1}{\kappa(1-\kappa)}\abs{\beta-\beta'}\frac{\min\bc{\beta,\beta'}pn}{\kappa pn} + \frac{1}{\kappa}\frac{\kappa-\max\bc{\beta,\beta'}}{1-\max\bc{\beta,\beta'}}\abs{\beta - \beta'} + \frac{\abs{\beta-\beta'}}{\kappa}\\
								 &\leq \frac{2}{\kappa(1-\kappa)}\abs{\beta-\beta'},
\end{align*}
where the last step uses the fact that $0 \leq \beta,\beta' \leq \kappa$. This tells us that
\[
\abs{\Delta(\w,\beta) - \Delta(\w,\beta')} \leq \frac{4 - \kappa}{\kappa(1-\kappa)}\abs{\beta-\beta'},
\]
which finishes the proof.
\end{proof}
\begin{proof}[Proof of Lemma~\ref{lem:same-beta-uc}]
We will prove the theorem by showing that the terms $A(\w,\beta)$ and $B(\w,\beta)$ exhibit uniform convergence. 

It is easy to see that $A(\w,\beta)$ exhibits uniform convergence since it is a simple average of population scores. The only thing to be taken care of is that $A(\w,\beta)$ contains $p$ in the normalization whereas $\hat A(\w,\beta)$ contains $\hat p$. However, since $p$ and $\hat p$ are very close with high probability, an argument similar to the one used in the proof of Theorem~\ref{lem:uc-proof-ramp-surrogate} can be used to conclude that with probability at least $1 - \delta$, we have
\[
\sup_{\w\in\W}\abs{A(\w,\beta) -\hat A(\w,\beta)} \leq \O{\sqrt{\frac{1}{b}\log\frac{1}{\delta}}}.
\]

To prove uniform convergence for $B(\w,\beta)$ we will use our earlier method of showing that this function exhibits pointwise convergence and that this function is Lipschitz with respect to $\w$. The Lipschitz property of $B(\w,\beta)$ is evident from an application of Corollary~\ref{cor:sup-lip-model}. To analyze its pointwise convergence property

Thus the function $B(\w,\beta)$, as analyzed in the proof of Lemma~\ref{lem:diff-beta-lip}, is composed by sorting the positives and negatives separately and taking the top few positions in each list and adding the scores present therein. This allows an application of Lemma~\ref{lem:rank-conv}, as used in the proof of Theorem~\ref{lem:uc-proof-ramp-surrogate}, separately to the positive and negative lists, to conclude the pointwise convergence bound for $B(\w,\beta)$.
\end{proof}
This concludes the proof of the uniform convergence bound for $\ell^\avg_\preckappa(\cdot)$.
\end{proof}

\subsection{Proof of Lemma~\ref{lem:sup-lip}}
\begin{replem}{lem:sup-lip}
Let $f_1,\ldots,f_m$ be $m$ real valued functions $f_i: \R^n \> \R$ such that every $f_i$ is $1$-Lipschitz with respect to the $\norm{\cdot}_\infty$ norm. Then the function
\[
g(\v) = \max_{i\in[m]}\ f_i(\v)
\]
is $1$-Lipschitz with respect to the $\norm{\cdot}_\infty$ norm too.
\end{replem}
\begin{proof}
Fix $\v,\v' \in \R^n$. The premise guarantees us that for any $i \in [m]$, we have
\[
\abs{f_i(\v) - f_i(\v')} \leq \norm{\v-\v'}_\infty.
\]
Now let $g(\v) = f_i(\v)$ and $g(\v') = f_j(\v')$. Then we have
\[
g(\v) - g(\v') = f_i(\v) - f_j(\v') \leq f_i(\v) - f_i(\v') \leq \norm{\v-\v'}_\infty,
\]
since $f_j(\v') \geq f_i(\v')$. Similarly we have $g(\v') - g(\v) \leq \norm{\v-\v'}_\infty$. This completes the proof.
\end{proof}

The following corollary would be most useful in our subsequent analyses.

\begin{cor}
\label{cor:sup-lip-model}
Let $\Psi: \W \rightarrow \R$ be a function defined as follows
\[
\Psi(\w) = \max_{\substack{\hy \in \bc{0,1}^n\\\norm{\hy}_1 = k}} \frac{1}{k}\sum \hy_i(\w^\top\x_i - c_i),
\]
where $c_i$ are constants independent of $\w$ and we assume without loss of generality that $\norm{\x_i}_2 \leq 1$ for all $i$. Then $\Psi(\cdot)$ is $1$- Lipschitz with respect to the $L_2$ norm i.e. for all $\w,\w'\in\W$
\[
\abs{\Psi(\w) - \Psi(\w')} \leq \norm{\w-\w'}_2.
\]
\end{cor}
\begin{proof}
Note that for any $\hy$ such that $\norm{\hy}_1 = k$, the function $f_{\hy}(\v) = \frac{1}{k}\sum \hy_i(\v_i - c_i)$ is $1$-Lipschitz with respect to the $\norm{\cdot}_\infty$ norm. Thus if we define
\[
\Phi(\v) = \max_{\norm{\hy}_1=k}\ f_{\hy}(\v),
\]
then an application of Lemma~\ref{lem:sup-lip} tells us that $\Phi(\cdot)$ is $1$-Lipschitz with respect to the $\norm{\cdot}_\infty$ norm as well. Also note that if we define
\[
\v(\w) = \br{\w^\top\x_1- c_1, \ldots, \w^\top\x_n - c_n},
\]
then we have
\[
\Psi(\w) = \Phi(\v(\w))
\]
We now note that by an application of Cauchy-Schwartz inequality, and the fact that $\norm{\x_i}_2 \leq 1$ for all $i$, we have
\[
\norm{\v(\w)-\v(\w')}_\infty \leq \norm{\w-\w'}_2
\]
Thus we have
\[
\abs{\Psi(\w) - \Psi(\w')} = \abs{\Phi(\v(\w)) - \Phi(\v(\w'))} \leq \norm{\v(\w)-\v(\w')}_\infty \leq \norm{\w-\w'}_2
\]
which gives us the desired result.
\end{proof}

\subsection{Proof of Lemma~\ref{lem:rank-conv}}
\begin{replem}{lem:rank-conv}
Let $\V$ be a universe with a total order $\succeq$ established on it and let $\v_1,\ldots,\v_n$ be a population of $n$ items arranged in decreasing order. Let $\hat \v_1,\ldots,\hat \v_b$ be a sample chosen i.i.d. (or without replacement) from the population and arranged in decreasing order as well. Then for any fixed $h: \V \> [-1,1]$ and $\kappa \in (0,1]$, we have, with probability at least $1 - \delta$ over the choice of the samples,
\[
\abs{\frac{1}{\ceil{\kappa n}}\sum_{i=1}^{\ceil{\kappa n}}h(\v_i) - \frac{1}{\ceil{\kappa b}}\sum_{i=1}^{\ceil{\kappa b}}h(\hat \v_i)} \leq 4\sqrt\frac{\log\frac{2}{\delta}}{\kappa b}
\]
\end{replem}
\begin{proof}
We will assume, for sake of simplicity, that $\kappa n$ and $\kappa b$ are both integers so that there are no rounding off issues. Let $\v^\ast_n := \v_{\kappa n}$ and $\v^\ast_b := \hat \v_{\kappa b}$ denote the elements at the bottom of the $\kappa$-th fraction of the top in the sorted population and sample lists (recall that the population and the sample lists are sorted in descending order). Also let $\T(\v) := \Ind{\v \succeq \v^\ast_n}$ and $\hat\T(\v) := \Ind{\v \succeq \v^\ast_b}$ (note that $\Ind{E}$ is the indicator variable for the event $E$) so that we have
\begin{align*}
\abs{\frac{1}{\kappa n}\sum_{i=1}^{\kappa n}h(\v_i) - \frac{1}{\kappa b}\sum_{i=1}^{\kappa b}h(\hat \v_i)} &= \abs{\frac{1}{\kappa n}\sum_{i=1}^{n}\T(\v_i)\cdot h(\v_i) - \frac{1}{\kappa b}\sum_{i=1}^{b}\hat\T(\hat \v_i)\cdot h(\hat \v_i)}\\
						&\leq \abs{\frac{1}{\kappa n}\sum_{i=1}^{n}\T(\v_i)\cdot h(\v_i) - \frac{1}{\kappa b}\sum_{i=1}^{b}\T(\hat \v_i)\cdot h(\hat \v_i)} + \abs{\frac{1}{\kappa b}\sum_{i=1}^{b}\br{\T(\hat \v_i)-\hat\T(\hat \v_i)}\cdot h(\hat \v_i)}\\
						&\leq 2\sqrt\frac{\log\frac{2}{\delta}}{\kappa b} + \underbrace{\abs{\frac{1}{\kappa b}\sum_{i=1}^{b}\br{\T(\hat \v_i)-\hat\T(\hat \v_i)}\cdot h(\hat \v_i)}}_{(A)},
\end{align*}
where the third step follows from Bernstein's inequality (which holds in situations with sampling without replacement as well \cite{Boucheron04concentrationinequalities}) since $\abs{\T(\v)\cdot h(\v)} \leq 1$ for all $\v$ and we have assumed $b \geq \frac{1}{\kappa}\log\frac{2}{\delta}$. Now if $\v^\ast_n \succeq \v^\ast_b$, then we have $\hat\T(\v) \geq \T(\v)$ for all $\v$. On the other hand if $\v^\ast_b \succeq \v^\ast_n$, then we have $\hat\T(\v) \leq \T(\v)$ for all $\v$. This means that since $\abs{h(\v)} \leq 1$ for all $\v$, we have
\[
(A) \leq \abs{\frac{1}{\kappa b}\sum_{i=1}^{b}\br{\T(\hat \v_i)-\hat\T(\hat \v_i)}} = \abs{\frac{1}{\kappa b}\sum_{i=1}^{b}\T(\hat \v_i)- 1} \leq 2\sqrt\frac{\log\frac{2}{\delta}}{\kappa b},
\]
where the second step follows since $\frac{1}{\kappa b}\sum_{i=1}^{b}\hat\T(\hat \v_i) = 1$ by definition and the last step follows from another application of Bernstein's inequality. This completes the proof.
\end{proof}

\subsection{A Uniform Convergence Bound for the $\ell^{\mx}_\preckappa(\cdot)$ Surrogate}
Having proved a generalization bound for the $\ell^\avg_\preckappa(\cdot)$ surrogate, we note that similar techniques, that involve partitioning the candidate label space into labels that have a fixed true positive rate $\beta$, and arguing uniform convergence for each partition, can be used to prove a generalization bound for the $\ell^{\mx}_\preckappa(\cdot)$ surrogate as well. We postpone the details of the argument to a later version of the paper.

\section{Proof of Theorem~\ref{thm:conv-ssgd@k}}
\label{app:thm-conv-ssgd@k-proof}
\begin{repthm}{thm:conv-ssgd@k}
Let $\bar\w$ be the model returned by Algorithm~\ref{algo:stoc-grad-preck} when executed on a stream with $T$ batches of length $b$. Then with probability at least $1 -\delta$, for any $\w^\ast \in \W$, we have
\[
\ell^\avg_\preckappa(\bar\w;\ZZ) \leq \ell^\avg_\preckappa(\w^\ast;\ZZ) + \O{\sqrt{\frac{1}{b}\log\frac{T}{\delta}}}+ \O{\sqrt\frac{1}{T}}
\]
\end{repthm}
\begin{proof}
The proof of this theorem closely follows that of Theorems 7 and 8 in \cite{KarN014}. More specifically, Theorem 6 from \cite{KarN014} ensures that any convex loss function demonstrating uniform convergence would ensure a result of the kind we are trying to prove. Since Theorem~\ref{thm:uc-preck-surrogates} confirms that $\ell^\avg_\preckappa(\cdot)$ exhibits uniform convergence, the proof follows.
\end{proof}


%% file: structsvm.tex
\section{Structural SVM Surrogate for \preck}
\label{sec:old-surrogate}

The structural SVM surrogate for \preck for a set of $n$ points $\{(\x_1, y_1), \ldots, (\x_n, y_n)\} \in (\R^d \times \{0, 1\})^n$ and model $w \in \R^d$  can be written as $\ell_{\preck}^{\text{struct}}(w)$:
\[
\max_{\substack{\widehat{\y} \in \{0,1\}^n\\\norm{\widehat \y}_1=k}} \bigg\{ 1 + \sum_{i=1}^n \widehat{\y}_i \bigg(\frac{1}{n}\w^\top \x_i - \frac{1}{k}\y_i\bigg) \,-\, \frac{1}{n}\sum_{i=1}^n {\y}_i \w^\top \x_i \bigg\}.
\]
We shall now give a simple setting where this surrogate produces a suboptimal model.

Consider a set of 6 points in $\R \times \{0, 1\}$: $\{(-1, 1), (-1, 1), (-2, 1), (-3, 0), (-3, 0), (-3, 0)\}$, and suppose we are interested in Prec@1. Note that the optimum model that maximizes prec@1 on these points has a positive sign. We will now show that the model $w^* \in \R$ that maximizes the above structural SVM surrogate on these points has a negative sign. On the contrary, let us assume that $w^*$ has a positive sign, and arrive at a contradiction; we shall consider the following two cases:

(i) $w^* > \frac{3}{2}$. It can be verified that
\begin{align*}
\ell_{\preck}^{\text{struct}}(w^*) &= 1 + \bigg(\frac{1}{6}(-w^*) -1\bigg) - \frac{1}{6}(-w^* + -w^* + -2w^*)\\
 &= \frac{1}{2}w^*
\end{align*}
On the other hand, for the model $w' = - w^*$, we have
\begin{align*}
\ell_{\preck}^{\text{struct}}(w') 
	&=
	 1 + \bigg(\frac{1}{6}(-3w') - 0\bigg) - \frac{1}{6}(-w' + -w' + -2w')\\
	&=	 
	 1 + \bigg(\frac{1}{6}(3w^*) - 0\bigg) - \frac{1}{6}(w^* + w^* + 2w^*)\\
	&=  1-\frac{1}{6}w^* ~<~ \ell_{\preck}^{\text{struct}}(w^*),
\end{align*}
where the last step follows from $w^* > \frac{3}{2}$; clearly, $w^*$ is not optimal for the structural SVM surrogate, and hence a contradiction.

(i) $w^* \leq \frac{3}{2}$. Here we have
\begin{align*}
\ell_{\preck}^{\text{struct}}(w^*) &= 1 + \bigg(\frac{1}{6}(-3w^*) - 0\bigg) - \frac{1}{6}(-w^* + -w^* + -2w^*)\\
 &= 1 + \frac{1}{6}w^*.
\end{align*}
For $w' = - w^*$, 
\begin{align*}
\ell_{\preck}^{\text{struct}}(w') 
	&=
	 1 + \bigg(\frac{1}{6}(-3w') - 0\bigg) - \frac{1}{6}(-w' + -w' + -2w')\\
	&=
	 1 + \bigg(\frac{1}{6}(3w^*) - 0\bigg) - \frac{1}{6}(w^* + w^* + 2w^*)\\
	&=  1-\frac{1}{6}w^* ~<~ \ell_{\preck}^{\text{struct}}(w^*).
\end{align*}
Here again, we have a contradiction. Notice that this surrogate can take negative values (when $w < - 6$ for example) whereas \preck is a positive valued function. This clearly indicates that this surrogate cannot upper bound \preck. More specifically, notice that for $w < 0$, we have $\preck(w) = 1$, however, the above analysis demonstrates cases when $\ell_{\preck}^{\text{struct}}(w) < 1$ which gives an explicit example that this surrogate is not even an upper bounding surrogate. 

%% file: app_exps.tex
\section{Additional Empirical Results}
\label{app:exps}

\begin{figure*}[h!]
\centering
\subfigure[KDD08]{
\includegraphics[scale=0.51]{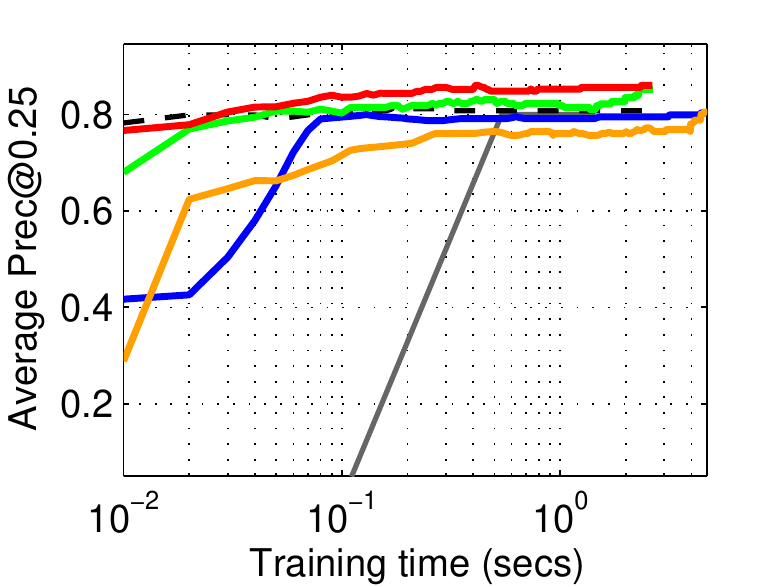}
\label{subfig:kdd08}
}
\subfigure[Covtype]{
\includegraphics[scale=0.51]{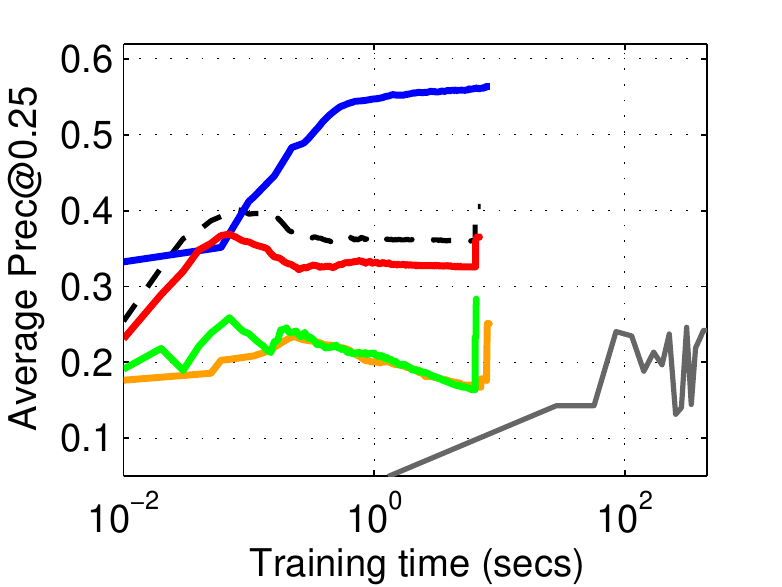}
\label{subfig:covtype}
}
\subfigure[Cod-RNA]{
\includegraphics[scale=0.51]{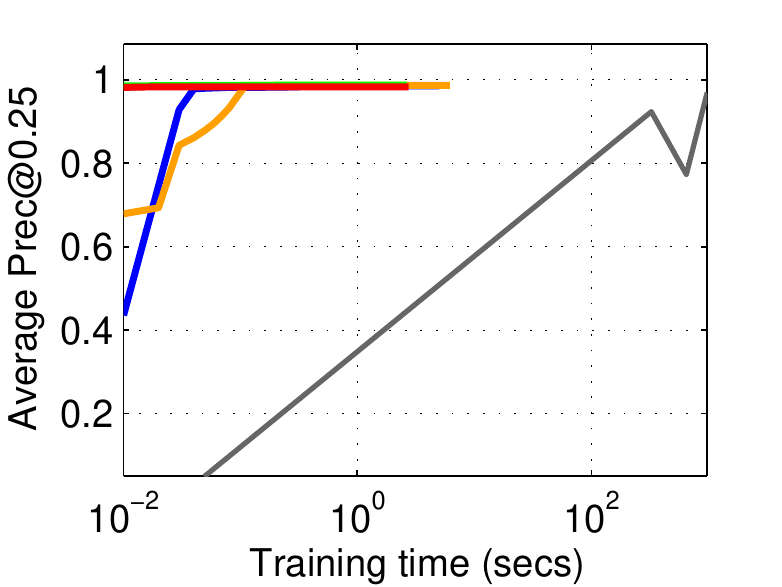}
\label{subfig:codrna}
}
\subfigure{
\includegraphics[scale=0.15]{Plots/0.25/legend.png}
}
\caption{A comparison of the proposed perceptron and SGD based methods with baseline methods (SVMPerf and \spmb) on prec@0.25 maximization tasks.}
\label{fig:app_training-time}
\end{figure*}